\documentclass{article}

\usepackage{microtype}
\usepackage{graphicx}
\usepackage{subfigure}
\usepackage{booktabs} 

\usepackage{silence}
\usepackage{amsmath,amsfonts,amssymb,amsthm,thmtools}
\usepackage{mathtools}
\usepackage{xparse}
\usepackage{enumitem} 
\usepackage{etoolbox}
\usepackage{amsfonts}
\usepackage{hhline}
\usepackage{multirow}
\usepackage{tabularx}

\usepackage[accepted]{icml2019}

\usepackage{hyperref}

\gdef\isarxiv{1}

\usepackage[capitalise,nameinlink,noabbrev]{cleveref}

\declaretheoremstyle[bodyfont=\itshape,notefont=\bfseries]{thmbf}
\declaretheoremstyle[notefont=\bfseries]{defibf}
\declaretheoremstyle[headfont=\normalfont\itshape,qed=$\square$]{proofita}

\declaretheorem[style=thmbf,numberwithin=section,name={Theorem}]{theorem}
\declaretheorem[style=thmbf,numberlike=theorem,name={Proposition}]{proposition}
\declaretheorem[style=thmbf,numberlike=theorem,name={Lemma}]{lemma}

\declaretheorem[style=thmbf,numbered=no,name={Corollary}]{corollary*}

\declaretheorem[style=defibf,numberlike=theorem,name={Example}]{example}
\declaretheorem[style=defibf,numberlike=theorem,name={Definition}]{definition}

\declaretheorem[style=proofita,numbered=no,name={Proof}]{demo}
\declaretheorem[style=defibf,numbered=no,name={Remark}]{remark}

\usepackage[all]{foreign}

\newcommand{\exprnn}{\textsc{exprnn}}
\newcommand{\rnn}{\textsc{rnn}}
\newcommand{\lstm}{\textsc{lstm}}
\newcommand{\gru}{\textsc{gru}}
\newcommand{\urnn}{\textsc{urnn}}
\newcommand{\eurnn}{\textsc{eurnn}}
\newcommand{\scornn}{\textsc{scornn}}
\newcommand{\scurnn}{\textsc{scurnn}}
\newcommand{\rgd}{\textsc{rgd}}
\newcommand{\mnist}{\textsc{mnist}}
\newcommand{\pmnist}{\textsc{p-mnist}}
\newcommand{\timit}{\textsc{timit}}
\newcommand{\stft}{\textsc{stft}}
\newcommand{\fft}{\textsc{fft}}
\newcommand{\mse}{\textsc{mse}}
\newcommand{\cnn}{\textsc{cnn}}
\newcommand{\rmsprop}{\textsc{rmsprop}}
\newcommand{\adam}{\textsc{adam}}
\newcommand{\adagrad}{\textsc{adagrad}}

\DeclarePairedDelimiter\abs{\lvert}{\rvert}     
\DeclarePairedDelimiter\pa{\lparen}{\rparen}    
\DeclarePairedDelimiter\cor{\lbrack}{\rbrack}   
\DeclarePairedDelimiter\norm{\lVert}{\rVert}    

\NewDocumentCommand{\infnorm}{ s O{} m }{%
  \IfBooleanTF{#1}{\norm*{#3}}{\norm[#2]{#3}}_{\infty}%
}
\NewDocumentCommand{\twonorm}{ s O{} m }{%
  \IfBooleanTF{#1}{\norm*{#3}}{\norm[#2]{#3}}_2%
}
\NewDocumentCommand{\tvnorm}{ s O{} m }{%
  \IfBooleanTF{#1}{\norm*{#3}}{\norm[#2]{#3}}_{\textup{TV}}%
}
\NewDocumentCommand{\onenorm}{ s O{} m }{%
  \IfBooleanTF{#1}{\norm*{#3}}{\norm[#2]{#3}}_1%
}
\NewDocumentCommand{\frobnorm}{ s O{} m }{%
  \IfBooleanTF{#1}{\norm*{#3}}{\norm[#2]{#3}}_F%
}
\NewDocumentCommand{\scalar}{s O{} >{\SplitArgument{1}{,}}m}{%
    \IfBooleanTF{#1}{\scalaraux*#3}{\scalaraux[#2]#3}%
}
\DeclarePairedDelimiterX{\scalaraux}[2]{\langle}{\rangle}{#1, #2}

\newcommand*\circledaux[1]{\tikz[baseline=(char.base)]{
    \node[shape=circle,draw,inner sep=0.8pt] (char) {#1};}}

\NewDocumentCommand{\circled}{ m o }{%
    \IfNoValueTF{#2}{ \circledaux{#1} }{ \stackrel{\circledaux{#1}}{#2} }%
}

\renewcommand*{\Im}{\mathrm{Im}}                        
\newcommand*\dif{\mathrm{d}}                            
\newcommand*\Id{\mathrm{Id}}                            
\newcommand\grad{\nabla}                                
\newcommand\conn{\nabla}                                
\newcommand*\defi{:=}                                   
\newcommand*\iso{\cong}                                 
\newcommand{\code}{\texttt}                             
\newcommand\deriv[1]{\frac{\mathrm{d}}{\mathrm{d}#1}}   

\newcommand*\C{\mathbb{C}}                              
\newcommand*\R{\mathbb{R}}                              
\newcommand*\Z{\mathbb{Z}}                              
\newcommand*\N{\mathbb{N}}                              

\let\epsilon\varepsilon
\let\subset\subseteq

\newcommand*\deffun[1]{\dodeffunction#1\relax}
\def\dodeffunction#1:#2->#3;#4\relax
    {\ifblank{#4}
    {#1\colon#2\to#3}
    {\!\begin{aligned}#1\colon#2&\to#3\\ \dodeffunctionaux#4\relax\end{aligned}}}
\def\dodeffunctionaux#1->#2\relax{#1&\mapsto#2}

\catcode`\|=\active
\DeclarePairedDelimiterX\set[1]{\lbrace}{\rbrace}
  {\mathcode`\|="8000 \def|{\:\delimsize\vert\:}#1}
\catcode`\|=12 %

\newcommand\trans[1]{#1^\intercal}                          
\DeclareMathOperator\tr{tr}                                 
\newcommand\E{\R^n}                                         
\newcommand\dual[1]{#1^\ast}                                
\newcommand\Mani{\mathcal{M}}                               
\renewcommand\S{\mathbb{S}}                                 
\DeclareMathOperator\gradmani{grad}                         
\DeclareMathOperator\ad{ad}                                 
\DeclareMathOperator\Ad{Ad}                                 
\newcommand\g{\mathfrak{g}}                                 
\DeclareMathOperator{\End}{End}                             
\newcommand\adj[1]{#1^\ast}                                 
\DeclareMathOperator{\unif}{\mathcal{U}}                    
\newcommand\I{\mathrm{I}}                                   

\NewDocumentCommand{\enorm}{ s O{} m }{%
    \IfBooleanTF{#1}{\norm*{#3}}{\norm[#2]{#3}}_{\E}%
}
\NewDocumentCommand{\denorm}{ s O{} m }{%
    \dual{\IfBooleanTF{#1}{\enorm*{#3}}{\enorm[#2]{#3}}}%
}

\DeclareMathOperator{\Uaux}{U}                           
\DeclareMathOperator{\GLaux}{GL}                         
\DeclareMathOperator{\Oaux}{O}                           
\DeclareMathOperator{\SOaux}{SO}                         
\DeclareMathOperator{\Skewaux}{Skew}                     
\NewDocumentCommand{\U}{ m }{ \Uaux\pa{#1} }

\NewDocumentCommand{\GL}{ m }{ \GLaux\pa{#1} }
\NewDocumentCommand{\Ort}{ m }{ \Oaux\pa{#1} }
\NewDocumentCommand{\SO}{ m }{ \SOaux\pa{#1} }
\NewDocumentCommand{\Skew}{ m }{ \Skewaux\pa{#1} }

\NewDocumentCommand{\M}{ >{\SplitArgument{1}{,}}m}{%
    \R^{\prodaux #1}%
}

\NewDocumentCommand{\commasaux}{ m m }{%
    \IfNoValueTF{#2}{ #1 }{ #1, #2 }%
}
\NewDocumentCommand{\prodaux}{ m m }{%
    \IfNoValueTF{#2}{ #1 \times #1 }{ #1 \times #2 }%
}

\makeatletter
\renewcommand\paragraph{\@startsection{paragraph}{4}{\z@}%
                                    {0ex \@plus0.5ex \@minus.2ex}%
                                    {-1em}%
                                    {\normalfont\normalsize\bfseries}}
\makeatother

\usepackage{todonotes}

\icmltitlerunning{Cheap Orthogonal Constraints in Neural Networks}

\begin{document}

\twocolumn[
\icmltitle{Cheap Orthogonal Constraints in Neural Networks:\\
           A Simple Parametrization of the Orthogonal and Unitary Group}


\icmlsetsymbol{equal}{*}

\begin{icmlauthorlist}
\icmlauthor{Mario Lezcano-Casado}{oxmaths}
\icmlauthor{David Mart\'inez-Rubio}{oxcs}
\end{icmlauthorlist}

\icmlaffiliation{oxmaths}{Mathematical Institute, University of Oxford, Oxford, United Kingdom}
\icmlaffiliation{oxcs}{Department of Computer Science, University of Oxford, Oxford, United Kingdom}
\icmlcorrespondingauthor{Mario Lezcano Casado}{mario.lezcanocasado@maths.ox.ac.uk}


\icmlkeywords{Optimization, Riemannian Geometry, Gradient Descent, Lie Groups, Recurrent Neural Networks}

\vskip 0.3in
]



\printAffiliationsAndNotice{} 

\begin{abstract}
We introduce a novel approach to perform first-order optimization with orthogonal and unitary constraints.
This approach is based on a parametrization stemming from Lie group theory through the \emph{exponential map}.
The parametrization transforms the constrained optimization problem into an unconstrained one over a Euclidean space, for which common first-order optimization methods can be used.
The theoretical results presented are general enough to cover the special orthogonal group, the unitary group and, in general, any connected compact Lie group.
We discuss how this and other parametrizations can be computed efficiently through an implementation trick, making numerically complex parametrizations usable at a negligible runtime cost in neural networks.
In particular, we apply our results to \rnn s with orthogonal recurrent weights, yielding a new architecture called \exprnn.
We demonstrate how our method constitutes a more robust approach to optimization with orthogonal constraints, showing faster, accurate, and more stable convergence in several tasks designed to test \rnn s.
\ifdefined\isarxiv
\footnote{Implementation can be found at \url{https://github.com/Lezcano/expRNN}}
\fi

\end{abstract}

\section{Introduction}
Training deep neural networks presents many difficulties. One of the most important is the exploding and vanishing gradient problem, as first observed and studied in~\cite{bengio1994learning}. This problem arises from the ill-conditioning of the function defined by a neural network as the number of layers increase. This issue is particularly problematic in Recurrent Neural Networks (\rnn s). In \rnn s the eigenvalues of the gradient of the recurrent kernel explode or vanish exponentially fast with the number of time-steps whenever the recurrent kernel does not have unitary eigenvalues~\cite{arjovsky2016unitary}. This behavior is the same as the one encountered when computing the powers of a matrix, and results in very slow convergence (vanishing gradient) or a lack of convergence (exploding gradient).

In the seminal paper~\cite{arjovsky2016unitary}, they note that unitary matrices have properties that would solve the exploding and vanishing gradient problems. These matrices form a group called the \emph{unitary group} and they have been studied extensively in the fields of Lie group theory and Riemannian geometry. Optimization methods over the unitary and orthogonal group have found rather fruitful applications in \rnn s in recent years (\cf{} \Cref{sec:related}).

In parallel to the work on unitary \rnn s, there has been an increasing interest for optimization over the orthogonal group and the Stiefel manifold in neural networks~\cite{harandi2016generalized,ozay2016optimization,huang2017orthogonal,bansal2018can}. As shown in these papers, orthogonal constraints in linear and \cnn{} layers can be rather beneficial for the generalization of the network as they act as a form of implicit regularization. The main problem encountered while using these methods in practice was that optimization with orthogonality constraints was neither simple nor computationally cheap. We aim to close that bridge.

In this paper we present a simple yet effective way to approach problems that present orthogonality or unitary constraints. We build on results from Riemannian geometry and Lie group theory to introduce a parametrization of these groups, together with theoretical guarantees for it.

This parametrization has several advantages, both theoretical and practical:
\begin{enumerate}[itemsep=0mm]
    \item It can be used with general purpose optimizers.
    \item The parametrization does not create additional minima or saddle points in the main parametrization region.
    \item It is possible to use a structured initializer to take advantage of the structure of the eigenvalues of the orthogonal matrix.
    \item Other approaches need to enforce hard orthogonality constraints, ours does not.
\end{enumerate}
Most previous approaches fail to satisfy one or many of these points. The parametrization in~\cite{helfrich18a} and~\cite{maduranga2018complexur} comply with most of these points but they suffer degeneracies that ours solves (\cf{} Remark in~\Cref{sec:approximations}). We compare our architecture with other methods to optimize over $\SO{n}$ and $\U{n}$ in the remarks in~\Cref{sec:theory,sec:exprnn}.

\paragraph{High-level idea}
The matrix exponential maps skew-symmetric matrices to orthogonal matrices transforming an optimization problem with orthogonal constraints into an unconstrained one. We use Pad\'e approximants and the scale-squaring trick to compute machine-precision approximations of the matrix exponential and its gradient. We can implement the parametrization with negligible overhead observing that it does not depend on the batch size.

\paragraph{Structure of the Paper}
In~\Cref{sec:theory}, we introduce the parametrization and present the theoretical results that support the efficiency of the exponential parametrization. In~\Cref{sec:exprnn}, we explain the implementation details of the layer. Finally, in~\Cref{sec:experiments}, we present the numerical experiments confirming the numerical advantages of this parametrization.

\section{Related Work}\label{sec:related}
\paragraph{Riemannian gradient descent.}
There is a vast literature on optimization methods on Riemannian manifolds, and in particular for matrix manifolds, both in the deterministic and the stochastic setting. Most of the classical convergence results from the Euclidean setting have been adapted to the Riemannian one~\cite{absil2009optimization,bonnabel2013stochastic,boumal2016global,zhang2016riemannian,sato2017riemannian}. On the other hand, the problem of adapting popular optimization algorithms like \rmsprop~\cite{tieleman2012lecture}, \adam~\cite{kingma2014adam} or \adagrad~\cite{duchi2011adaptive} is a topic of current research~\cite{Roy_2018_CVPR,becigneul2018riemannian}.

\paragraph{Optimization over the Orthogonal and Unitary groups.}
The first formal study of optimization methods on manifolds with orthogonal constraints (Stiefel manifolds) is found in the thesis~\cite{Smith:1993:GOM:165579}. These ideas were later simplified in the seminal paper~\cite{edelman1998geometry}, where they were generalized to Grassmannian manifolds and extended to get the formulation of the conjugate gradient algorithm and the Newton method for these manifolds. After that, optimization with orthogonal constraints has been a central topic of study in the optimization community. A rather in depth literature review of existing methods for optimization with orthogonality constraints can be found in~\cite{jiang2015framework}. When it comes to the unitary case, the algorithms used in practice are similar to those used in the real case, \cf~\cite{manton2002optimization,abrudan2008steepest}.

\paragraph{Unitary \rnn s.}
The idea of parametrizing the matrix that defines an RNN by a unitary matrix was first proposed in~\cite{arjovsky2016unitary}. Their parametrization centers on a matrix-based fast Fourier transform-like (\fft) approach. As pointed out in~\cite{jing2017tunable}, this representation, although efficient in memory, does not span the whole space of unitary matrices, giving the model reduced expressiveness. This second paper solves this issue in the same way it is solved when computing the \fft---using $\log(n)$ iterated butterfly operations. A different approach to perform this optimization was presented in~\cite{wisdom2016full,vorontsov2017orthogonality}. Although not mentioned explicitly in either of the papers, this second approach consists of a retraction-based Riemannian gradient descent via the Cayley transform. The paper~\cite{hyland2017learning} proposes to use the exponential map on the complex case, but they do not perform an analysis of the algorithm or provide a way to approximate the map nor the gradients. A third approach has been presented in~\cite{mhammedi2017efficient} via the use of Householder reflections. Finally, in~\cite{helfrich18a} and the follow-up~\cite{maduranga2018complexur}, a parametrization of the orthogonal group via the use of the Cayley transform is proposed. We will have a closer look at these methods and their properties in~\Cref{sec:theory,sec:exprnn}.

\section{Parametrization of Compact Lie Groups}\label{sec:theory}
For a much broader introduction to Riemannian geometry and Lie group theory see ~\Cref{sec:appendix_background}. We will restrict our attention to the special orthogonal~\footnote{Note that we consider just matrices with determinant equal to one, since the full group of orthogonal matrices $\Ort{n}$ is not connected, and hence, not amenable to gradient descent algorithms.} and unitary case, but the results in this section can be generalized to any connected compact matrix Lie group equipped with a bi-invariant metric. We prove the results for general connected compact matrix Lie groups in~\Cref{sec:appendix_comparison}.

\subsection{The Lie algebras of \texorpdfstring{$\SO{n}$}{SO(n)} and \texorpdfstring{$\U{n}$}{U(n)}}~\label{sec:lie_exponential_map}
We are interested in the study of parametrizations of the special orthogonal group
\[
    \SO{n} = \set{B \in \M{n} | \trans{B}B = \I,\, \det\pa{B} = 1}
\]
and the unitary group
\[
    \U{n} = \set{B \in \C^{n \times n} | \adj{B}B = \I}.
\]
These two sets are compact and connected Lie groups. Furthermore, when seen as submanifolds of $\M{n}$ (resp.\ $\C^{n \times n}$) equipped with the metric induced from the ambient space $\scalar{X, Y} = \tr\pa{\trans{X}Y}$ (resp.\ $\tr\pa{\adj{X}Y}$), they inherit a \emph{bi-invariant metric}, meaning that the metric is invariant with respect to left and right multiplication by matrices of the group. This is clear given that the matrices of the two groups are isometries with respect to the metric on the ambient space.

We call the tangent space at the identity element of the group the \emph{Lie algebra} of the group. For the two groups of interest, their Lie algebras are given by
\begin{gather*}
    \mathfrak{so}(n) = \set{A \in \M{n} | A + \trans{A} = 0},\\
    \mathfrak{u}(n) = \set{A \in \C^{n \times n} | A + \adj{A} = 0},
\end{gather*}
That is, the skew-symmetric and the skew-Hermitian matrices respectively. Note that these two spaces are isomorphic to a vector space. For example, for $\mathfrak{so}\pa{n}$, the isomorphism is given by
\[
    \deffun{\alpha : \R^{\frac{n\pa{n-1}}{2}} -> \mathfrak{so}\pa{n} ; A -> A - \trans{A}}
\]
where we identify $A \in \R^{\frac{n(n-1)}{2}}$ with an upper triangular matrix with zeros in the
diagonal.

\subsection{Parametrizing \texorpdfstring{$\SO{n}$}{SO(n)} and \texorpdfstring{$\U{n}$}{U(n)}}
In the theory of Lie groups there exists a tight connection between the structure of the Lie algebra and the geometry of the Lie group. One of the most important tools that is used to study one in terms of the other is the \emph{Lie exponential map}. The Lie exponential map on matrix Lie groups with a bi-invariant metric is given by the exponential of matrices. If we denote the group by $G$ (which would be $\SO{n}$ or $\U{n}$ in this case) and its Lie algebra by $\g$, we have the mapping $\exp : \g \to G$ defined as
\[
     \exp(A) \defi \I + A + \tfrac{1}{2}A^2 + \cdots
\]
This mapping is not surjective in general. On the other hand, there are particular families of Lie groups in which the exponential map is, in fact, surjective. Compact Lie groups are one of such families.
\begin{theorem}\label{thm:surjective}
    The Lie exponential map on a connected, compact Lie group is surjective.
\end{theorem}
\begin{proof}
We give a short self-contained proof of this classical result in~\Cref{sec:appendix_comparison}. We give an alternative, less abstract proof of this fact for the groups $\SO{n}$ and $\U{n}$ as a corollary of~\Cref{prop:bianalytic} in~\Cref{sec:appendix_fundamental_domain}.
\end{proof}

Both $\SO{n}$ and $\U{n}$ are compact and connected, so this result applies to them. As such, the exponential of matrices gives a complete parametrization of these groups.

\subsection{From Riemannian to Euclidean optimization}
In this section we describe some properties of the exponential parametrization which make it a sound choice for optimization with orthogonal constraints in neural networks.

Fix $G$ to be $\SO{n}$ or $\U{n}$ equipped with the metric\footnote{Note that in the real case we have that $\adj{A} = \trans{A}$.} $\scalar{X,Y} = \tr\pa{\adj{X}Y}$ and let $\g$ be its Lie algebra (the space of skew-symmetric or skew-Hermitian matrices). The exponential parametrization satisfies the following properties.

\paragraph{It can be used with general purpose optimizers.}
The exponential parametrization allows us to \emph{pullback} an optimization problem from the group $G$ back to the Euclidean space. If we have a problem
\begin{equation}\label{eq:min_riemann}
     \min_{B \in G} f(B)
\end{equation}
this is equivalent to solving
\begin{equation}\label{eq:min_euclid}
    \min_{A \in \g} f(\exp(A)).
\end{equation}
We noted in~\Cref{sec:lie_exponential_map} that $\g$ is isomorphic to a Euclidean vector space, and as such we can use regular gradient descent optimizers like \adam{} or \adagrad{} to approximate a solution to problem~\eqref{eq:min_euclid}.

A rather natural question to ask is whether using gradient-based methods to approximate the solution of problem~\eqref{eq:min_euclid} would give a sensible solution to problem~\eqref{eq:min_riemann}, given that precomposing with the exponential map might change the geometry of the problem. If the parametrization is, for example not locally unique, this might degrade the gradient flow and affect the performance of the gradient descent algorithm. In this section we will show theoretically that this parametrization has rather desirable properties for a parametrization of a manifold. We will confirm that these properties have a positive effect on the convergence of the gradient descent algorithms when compared with other parametrizations when applied to real problems in~\Cref{sec:experiments}.

\paragraph{It does not change the minimization problem.}
It is clear that a minimizer $\widehat{B}$ for problem~\eqref{eq:min_riemann} and a minimizer $\widehat{A}$ for problem~\eqref{eq:min_euclid} will be related by the equation $\widehat{B} = \exp(\widehat{A})$, since the exponential map is surjective, so if we find a solution to the second problem we will have a solution to the first one.

\paragraph{It acts as a change of metric on the group.}
If the parametrization did not induce a change of metric on the manifold it could mean that it would induce saddle points, which would potentially slow down the convergence of the optimization algorithm.

A map $\deffun{\phi : \mathcal{M} -> \mathcal{N};}$ with $\mathcal{N}$ a Riemannian manifold induces a metric on a differentiable manifold $\mathcal{M}$ whenever it is an immersion, that is, its differential is injective. The Lie exponential is not just an immersion, it is bi-analytic on an open neighborhood around the origin. The image of this neighborhood is sufficiently large to cover almost all the Lie group.
\begin{proposition}\label{prop:bianalytic}
    Let $G$ be $\SO{n}$ or $\U{n}$. The exponential map is analytic, invertible, with analytic inverse on a bounded open neigborhood $V$ of the origin and $\exp(V)$ covers almost all $G$ in the sense that the whole group lies in the closure of $\exp(V)$.
\end{proposition}
\begin{demo}
    See \Cref{sec:appendix_fundamental_domain}.
\end{demo}
This proposition assures that, as long as the optimization problem stays in the neighborhood $V$,
the representation of the matrices in $G$ is unique, so this parametrization is not creating
spurious minima. Furthermore, given that it is a diffemorphism, it is not creating saddle points on
$V$ either. Additionally, on this neighborhood, we have the adjoint of $\dif \exp$ with
respect to the metric, that is,
\[
    \scalar{\dif \exp (X), Y} = \scalar{X, \dif \exp^\ast (Y)}.
\]
This is the map that induces the new metric on $G$, through the pushforward of the canonical metric
from the Lie algebra into the Lie group. As such, the optimization process using our
parametrization can be seen as Riemannian gradient descent using this new metric, and all the
existent results developed for optimization over manifolds apply to this setting.

\begin{remark}
We saw empirically that whenever the initialization of the skew-symmetric matrix starts in $V$, the optimization path throughout all the training epochs does not leave $V$. For this reason, in practice the exponential parametrization behaves as a change of metric on the Lie group.
\end{remark}

\paragraph{The induced metric is different to the classic one.}
The standard first order optimization technique to solve problem~\eqref{eq:min_riemann} is given by Riemannian gradient descent~\cite{absil2009optimization}. In the Riemannian setting, we have the \emph{Riemannian exponential map} $\exp_B$ which maps lines that pass through the origin on the tangent space $T_B G$ to geodesics on $G$ that pass through $B$. In the special orthogonal or unitary case, when we choose the metric induced by the canonical metric on the ambient space, for a function defined on the ambient space, this translates to the update rule
\[
    B \gets B\exp\pa{-\eta\adj{B}\gradmani f(B)}
\]
for a learning rate $\eta > 0$, where $\exp$ is the exponential of matrices and $\gradmani f(B)$ denotes the gradient of the function restricted to $G$. We deduce this formula in~\Cref{ex:riemannian_gradient_descent_son}.

Computing the Riemannian exponential map exactly is computationally expensive in many practical situations. For this reason, approximations are in order. Retractions are of particular interest.
\begin{definition}[Retraction]
A retraction $r$ for a manifold $\Mani$ is defined as a family of functions $\deffun{r_x : T_x \Mani -> \Mani;}$ for every $x \in \Mani$ such that
\[
    r_x(0) = x \qquad \text{and} \qquad \pa{\dif r_x}_0 = \Id.\qedhere
\]
\end{definition}
In other words, retractions are a first order approximation of the Riemannian exponential map. A study of the convergence properties of first and second-order optimization algorithms when using retractions can be found in~\cite{boumal2016global}. In the case of $G$, we have that a way to form retractions is to choose a function $\deffun{\phi : \g -> G;}$ such that it is a first order approximation of the exponential of matrices and its image lies in $G$. Then, the update rule is given by
\[
    B \gets B\phi\pa{-\eta\adj{B}\gradmani f(B)}.
\]

\begin{remark}
    For the special orthogonal and the unitary group, one such function is the \emph{Cayley map}
    \[
        \phi(A) = \pa{\I + \tfrac{1}{2}A}\pa{\I - \tfrac{1}{2}A}^{-1}.
    \]
    This justifies theoretically the optimization methods used in~\cite{wisdom2016full,vorontsov2017orthogonality} and extends their work, given that all their architectures can still be applied with different retractions for these manifolds. In~\Cref{sec:approximations} we give examples of more involved retractions, and in~\Cref{sec:cheap} we explain why it is computationally cheap to use machine-accuracy approximants to compute the exponential map both in our approach and in the Riemannian gradient descent approach. Examples of other retractions and a deeper treatment of these objects can be found in~\Cref{sec:retractions}.
\end{remark}

The update rule for the exponential parametrization induces a retraction-like map for $A \in \g$
\[
    e^A \gets \exp\pa{A - \eta\grad\pa{f \circ \exp}(A)},
\]
where the gradient is the gradient with respect to the Euclidean metric, that is, the regular gradient, given that $f \circ \exp$ is defined on a Euclidean space.
A natural question that arises is whether this new update rule defines a retraction. It turns out that this map is not a retraction for $\SO{n}$ or $\U{n}$.
\begin{proposition}
    The step-update map induced by the exponential parametrization is not a retraction for $\SO{n}$ if $n > 2$ nor for $\U{n}$ if $n > 1$.
\end{proposition}
\begin{demo}
    It is a corollary of~\Cref{thm:retract}, where we give necessary and sufficient conditions for this map to be a retraction when defined on a compact, connected matrix Lie group.
\end{demo}

This tells us that the metric induced by the $\log$ map on $\SO{n}$ and $\U{n}$ is intrinsically different to the canonical metric on these manifolds when seen as submanifolds of $\R^{n \times n}$ (resp.\ $\C^{n \times n}$). In particular, it changes the geodesic flow defined by the metric.

\section{Numerical Implementation}\label{sec:exprnn}
As an application of this framework we show how to model an orthogonal (or unitary) recurrent neural network with it, that is, an \rnn{} whose recurrent matrix is orthogonal (or unitary). We also show how to implement numerically the ideas of the last section.

\subsection{Exponential \rnn{} Architecture}
Given a sequence of inputs $\pa{x_t} \subset \R^d$, we define an orthogonal exponential \rnn{} (\exprnn) with hidden size $p > 0$ as
\[
    h_{t+1} = \sigma\pa{\exp(A) h_t + Tx_{t+1}}
\]
where $A \in \Skew{p}$, $T \in \M{p,d}$, and $\sigma$ is some fixed non-linearity. In our experiments we chose the \code{modrelu} nonlinearity, as introduced in~\cite{arjovsky2016unitary}. Note that generalizing this architecture to the complex unitary case simply accounts for considering $A$ to be skew-Hermitian rather than skew-symmetric. We stayed with the real case because we did not observe any improvement in the empirical results when using the complex case.

\subsection{Approximating the exponential of matrices}\label{sec:approximations}
There is a myriad of methods to approximate the exponential of a matrix~\cite{moler2003nineteen}. Riemannian gradient descent over $\SO{n}$ requires that the result of the approximation is orthogonal. If not, the error would accumulate after each step making the resulting matrix deviate from the orthogonality constraint exponentially fast in the number of steps. On the other hand, the approximation of the exponential in our parametrization does not require orthogonality. This allows many other approximations of the exponential function. The requirement is removed because the orthogonal matrix is implicitly represented as an exponential of a skew-symmetric matrix. The loss of orthogonality in Riemannian gradient descent is due to storing an orthogonal matrix and updating it directly.

\paragraph{Pad\'e approximants.}
Pad\'e approximants are rational approximations of the form $\exp(A) \approx p_n(A)q_n(A)^{-1}$ for polynomials $p_n, q_n$ of degree $n$. A Pad\'e approximant of degree $n$ agrees with the Taylor expansion of the exponential to degree $2n$. The Cayley transform is the Pad\'e approximant of degree $1$. These methods and their implementations are described in detail in~\cite{higham2009scaling}.

\paragraph{Scale-squaring trick.}
The error of the Pad\'e approximant scales as $\mathcal{O}\pa{\norm{A}^{2n+1}}$. If $\norm{A} > 1$ and we have an approximant $\phi$, the scale-squaring trick accounts for computing $\phi\pa{\frac{A}{2^k}}^{2^k}$ for the first $k\in\N$ such that $\frac{\norm{A}}{2^k} < \frac{1}{2}$.
Most types of approximants, like Pad\'e's or a truncated Taylor expansion of the exponential, can be coupled with the scale-squaring trick to reduce the error~\cite{higham2009scaling}.

\begin{remark}\label{remark:cayley}
    Given that the Cayley transform is a degree $1$ Pad\'e approximant of the exponential, if we choose this approximant without the scale-squaring trick we essentially recover the parametrization proposed in~\cite{helfrich18a}. The Cayley transform suffers from the fact that, if the optimum has $-1$ as an eigenvalue, the weights of the corresponding skew-symmetric matrix will tend to infinity. The parametrization in ~\cite{helfrich18a} is the Cayley transform multiplied by a diagonal matrix $D$, but the parametrization still has the same problem, it just moves it to a different eigenvalue. \Cref{prop:bianalytic} assures that the exponential parametrization does not suffer from this problem.

    The follow-up work~\cite{maduranga2018complexur} mitigates this problem learning the diagonal of $D$ as well, but by doing so it loses the local unicity of the parametrization. \Cref{prop:bianalytic} assures that the exponential parametrization is not only locally unique, but also differentiable with differentiable inverse, thus inducing a metric.
\end{remark}

\begin{remark}
    It is straightforward to show that a degree $n$ Pad\'e approximant combined with the scaling-squaring trick also maps the skew-symmetric matrices to the special orthogonal matrices. This constitutes a much more precise retraction than the Cayley map at almost no extra computational cost. This observation can be used to improve the precision of the method proposed in~\cite{wisdom2016full,vorontsov2017orthogonality}.
\end{remark}

\paragraph{Exact approximation.} Combining the methods above we can get an efficient approximation of the exponential to machine-precision. The best one known to the authors is based on the paper~\cite{al2009new}. It accounts for an efficient use of the scaling-squaring trick and a Pad\'e approximant. This is the algorithm that we use on the experiments section to approximate the exponential.

\paragraph{Exact gradients.} A problem often encountered in practice is that of biased gradients. Even though an approximation might be good, it can significantly bias the gradient. An example of this would be trying to approximate the function $f \equiv 0$ on $[0, 1]$ by the functions $f_n(x) = \frac{\sin(2\pi nx)}{n}$. Even though $f_n \to f$ uniformly, their derivatives do not converge to zero. This problem is rather common when using involved parametrizations, for example, those coming from Chebyshev polynomials. On the other hand, the gradient can be implemented separately using a machine-precision formula.
\begin{proposition}\label{prop:grad_son}
    Let $A \in \Skew{n}$. For a function $\deffun{f : \R^{n \times n} -> \R;}$, denote the matrix $B = e^A$. We have that
    \[
        \grad\pa{f \circ \exp}(A) = B \pa{\dif \exp}_{-A}\pa{\tfrac{1}{2}\pa{\trans{\grad f(B)}B - \trans{B}\grad f(B)}}.
    \]
\end{proposition}
\begin{proof}
    It follows from the discussion in~\Cref{ex:riemannian_gradient_descent_son} and~\Cref{prop:gradient} in the supplementary material.
\end{proof}

The differential of the exponential of matrices $\pa{\dif \exp}_A$ can be approximated to machine-precision computing the exponential of a $2n \times 2n$ matrix~\cite{al2009computing}. We use this algorithm in conjunction with~\Cref{prop:grad_son} to implement the gradients.

\subsection{Parametrizations are computationally cheap.}\label{sec:cheap}
At first sight, one may think that an exact computation of the exponential and its gradient in neural networks is rather expensive. This is not the case when the exponential is just used as a parametrization. The value---and hence the gradient---of a parametrization does not depend on the training examples used to compute the stochastic gradient. For this reason, in order to compute the gradient of a function $\grad\pa{f \circ \phi}(A)$ with $B = \phi(A)$, we can first let the auto-differentiation engine compute the stochastic gradient of $f$ with respect to $B$, that is, $\grad f(B)$. The value $\grad f(B)$ depends on the batch size and the number of appearances of $B$ as a subexpression in the neural network (think of a recurrent kernel in an \lstm). We can use $\grad f(B)$ to compute---just once per batch---the gradient $\grad \pa{f \circ \phi}(A)$, for example with the formula given in~\Cref{prop:grad_son} for $\phi = \exp$. This allows the user to implement rather complex parametrizations, like the one we showed, without a noticeable runtime penalty. For instance, for an \rnn{} with batch size $b$, sequences of average length $\ell$, and a hidden size of $n$, in each iteration one needs to compute $bln$ matrix-vector products at the cost of $O(n^2)$ operations each. The overhead incurred using the exponential parametrization is the computation of two matrix exponentials that run in $O(n^3)$, which is negligible in comparison. In practice, with an \exprnn{} of size $512$, we did not observe any noticeable time penalty when using this parametrization trick with respect to not imposing orthogonality constraints at all.

\subsection{Intialization}
For the initialization of the layer with a matrix $A_0 \in \Skew{p}$, we drew ideas from both~\cite{henaff2016recurrent} and~\cite{helfrich18a}. Both initializations sample blocks of the form
\[
\begin{pmatrix}
    0 & s_i \\
    -s_i & 0
\end{pmatrix}.
\]
for $s_i$ i.i.d.\ and then form $A_0$ as a block-diagonal matrix with these blocks.

The \emph{Henaff} initialization consists of sampling $s_i \sim \unif\cor{-\pi, \pi}$. This defines a block-diagonal orthogonal matrix $e^A$ with uniformly distributed blocks on the corresponding torus of block-diagonal $2 \times 2$ rotations. We sometimes found that the sampling presented in~\cite{helfrich18a} performed better. This initialization, which we call \emph{Cayley}, accounts for sampling $u_i \sim \unif\cor{0, \frac{\pi}{2}}$ and then setting $s_i = -\sqrt{\frac{1 - \cos(u_i)}{1+\cos(u_i)}}$, thus biasing the eigenvalues towards $0$.

We chose as the initial vector $h_0 = 0$ for simplicity, as we did not observe any empirical improvement when using the initialization given in~\cite{arjovsky2016unitary}.

\section{Experiments}\label{sec:experiments}
In this section we compare the performance of our parametrization for orthogonal \rnn s with the following approaches:
\ifdefined\isarxiv
    \begin{itemize}[itemsep=0mm]
        \item Long short-term memory (\lstm). \\ \cite{hochreiter1997long}
        \item Unitary \rnn{} (\urnn). \\ \cite{arjovsky2016unitary}
        \item Efficient Unitary \rnn{} (\eurnn). \\ \cite{jing2017tunable}
        \item Cayley Parametrization (\scornn). \\ \cite{helfrich18a}
        \item Riemannian Gradient Descent (\rgd). \\ \cite{wisdom2016full}
    \end{itemize}
\else
Long short-term memory (\lstm), Unitary \rnn{} (\urnn), Efficient Unitary \rnn{} (\eurnn), Cayley Parametrization (\scornn), Riemannian Gradient Descent (\rgd) which can be found in \cite{hochreiter1997long}, \cite{arjovsky2016unitary}, \cite{jing2017tunable}, \cite{helfrich18a} and \cite{wisdom2016full} respectively.
\fi

\begin{figure*}[!tbp]
\vskip 0.2in
\begin{center}
  \begin{minipage}[b]{\columnwidth}
      \centerline{\includegraphics[width=0.83\columnwidth]{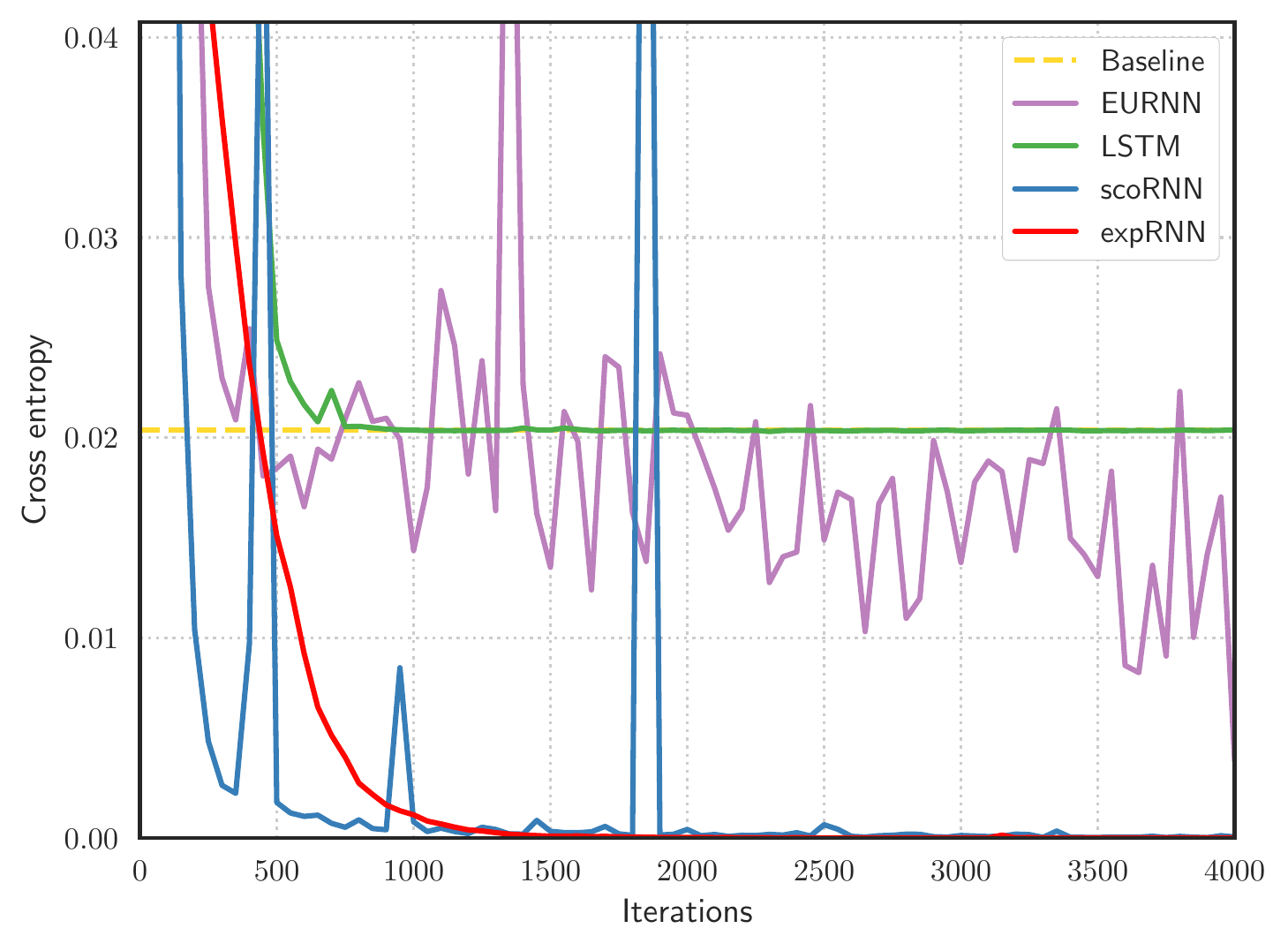}}
  \end{minipage}
  \begin{minipage}[b]{\columnwidth}
      \centerline{\includegraphics[width=0.83\columnwidth]{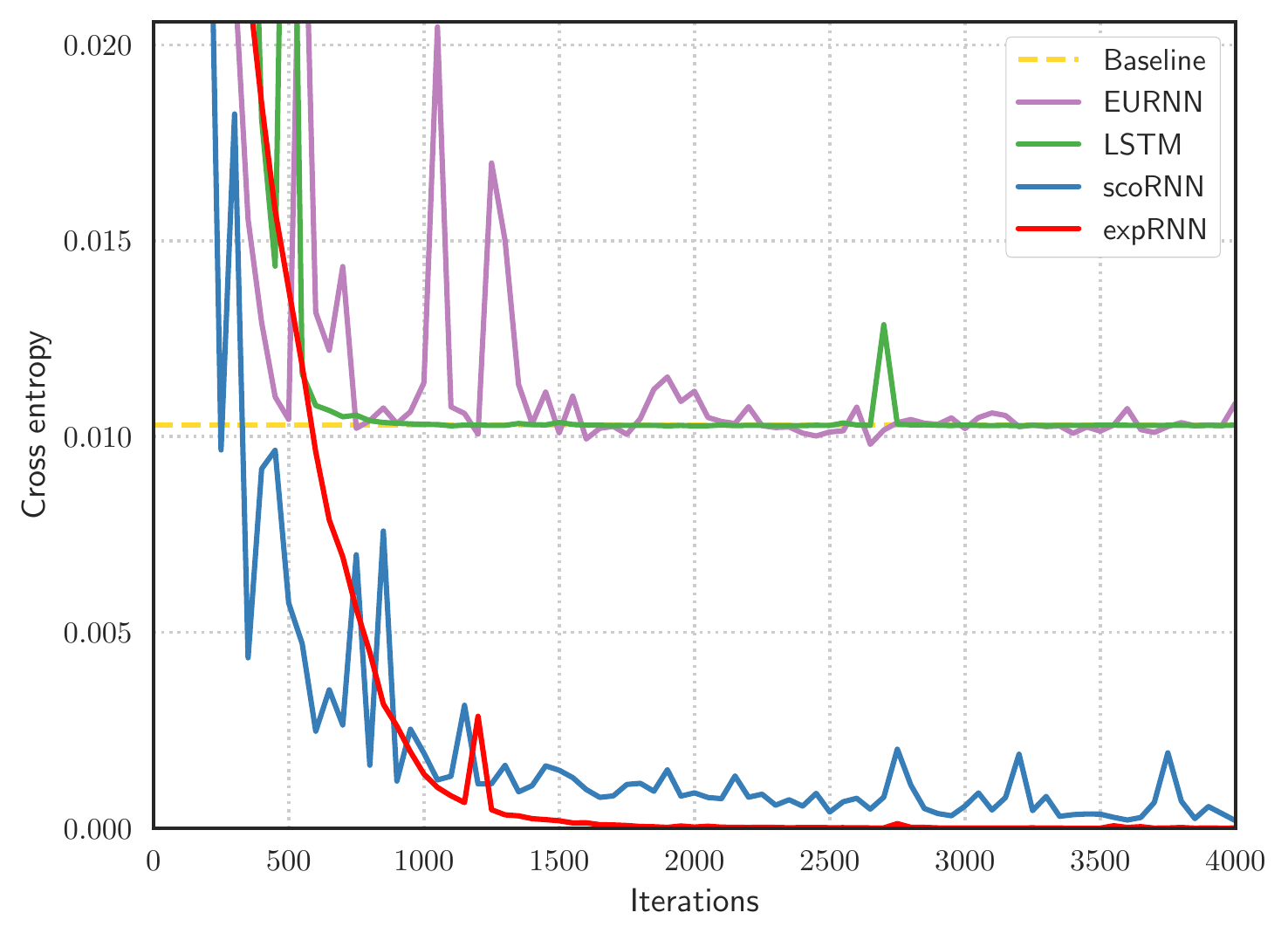}}
  \end{minipage}
  \caption{Cross entropy of the different algorithms in the copying problem for $L=1000$ (left) and $L=2000$ (right).}
  \label{fig:exp1}
\end{center}
\vskip -0.2in
\end{figure*}

We use three tasks that have become standard to measure the performance of \rnn s and their ability to deal with long-term dependencies. These are the copying memory task, the pixel-permuted \mnist{} task, and the speech prediction on the \timit{} dataset~\cite{arjovsky2016unitary,wisdom2016full,henaff2016recurrent,mhammedi2017efficient,helfrich18a}.

In~\Cref{sec:hyperparameters}, we enumerate the hyperparameters used for the experiments. The sizes of the hidden layer were chosen to match the number of learnable parameters of the other architectures.

\begin{remark}
    We found empirically that having a learning rate for the orthogonal parameters that is $10$ times larger than that of the non-orthogonal parameters yields a good performance in practice.
\end{remark}

For the other experiments, we executed the code that the other authors provided with the best hyperparameters that they reported and a batch of $128$. The results for \eurnn{} are those reported in~\cite{jing2017tunable}, and for \rgd{} and \urnn{} are those reported in~\cite{helfrich18a}.

The code with the exact configuration and seeds to replicate these results, and a plug-and-play implementation of \exprnn{} and the exponential framework can be found in
\ifdefined\isgithub
\begin{center}
\url{https://github.com/Lezcano/expRNN}
\end{center}
\else
\url{https://github.com/Lezcano/expRNN}.
\fi

\subsection{Copying memory task}
The copying memory task was first proposed in~\cite{hochreiter1997long}.
The task can be defined as follows. Let $\mathcal{A} = \set{a_k}_{k=1}^N$ be an alphabet and let $\code{<blank>}$, $\code{<start>}$ be two symbols not contained in $\mathcal{A}$. For a sequence length of $K$ and a spacing of length $L$, the input sequence would be $K$ ordered characters $\pa{b_k}_{k=1}^K$ sampled i.i.d.\ uniformly at random from $\mathcal{A}$, followed by $L$ repetitions of the character $\code{<blank>}$, the character $\code{<start>}$ and finally $K-1$ repetitions of the character $\code{<blank>}$ again. The output for this sequence would be $K+L$ times the $\code{<blank>}$ character and then the sequence of characters $\pa{b_k}_{k=1}^K$. In other words, the system has to recall the initial $K$ characters and reproduce them after detecting the input of the character $\code{<start>}$, which appears $L$ time-steps after the end of the input characters. For example, for $N = 4$, $K = 5$, $L = 10$, if we represent $\code{<blank>}$ with a dash and $\code{<start>}$ with a colon, and the alphabet $\mathcal{A} = \set{1, \dots, 4}$, the following sequences could be an element of the dataset:
\begin{align*}
    \text{Input:  }&\texttt{14221----------:----} \\
    \text{Output: }&\texttt{---------------14221}
\end{align*}

The loss function for this task is the cross entropy. The standard baseline for this task is the output of $K+L$ $\code{<blank>}$ symbols, followed by the remaining $K$ symbols being output at random. This strategy yields a cross entropy of $K\log\pa{N}/ \pa{L + 2K}$.

We observe that the training of \scornn{} is unstable, which is probably due to the degeneracies explained in the remark in~\Cref{sec:exprnn}.
In the follow-up paper~\cite{maduranga2018complexur}, \scurnn{} presents the same instabilities as its predecessor.
As explained in~\Cref{sec:exprnn}, \exprnn{} does not suffer of this, and can be observed in our experiments as a smoother convergence.
In the more difficult problem, $L=2000$, \exprnn{} is the only architecture that is able to fully converge to the correct answer.

\begin{figure*}[!tbp]
\vskip 0.2in
\begin{center}
  \begin{minipage}[b]{\columnwidth}
      \centerline{\includegraphics[width=0.83\columnwidth]{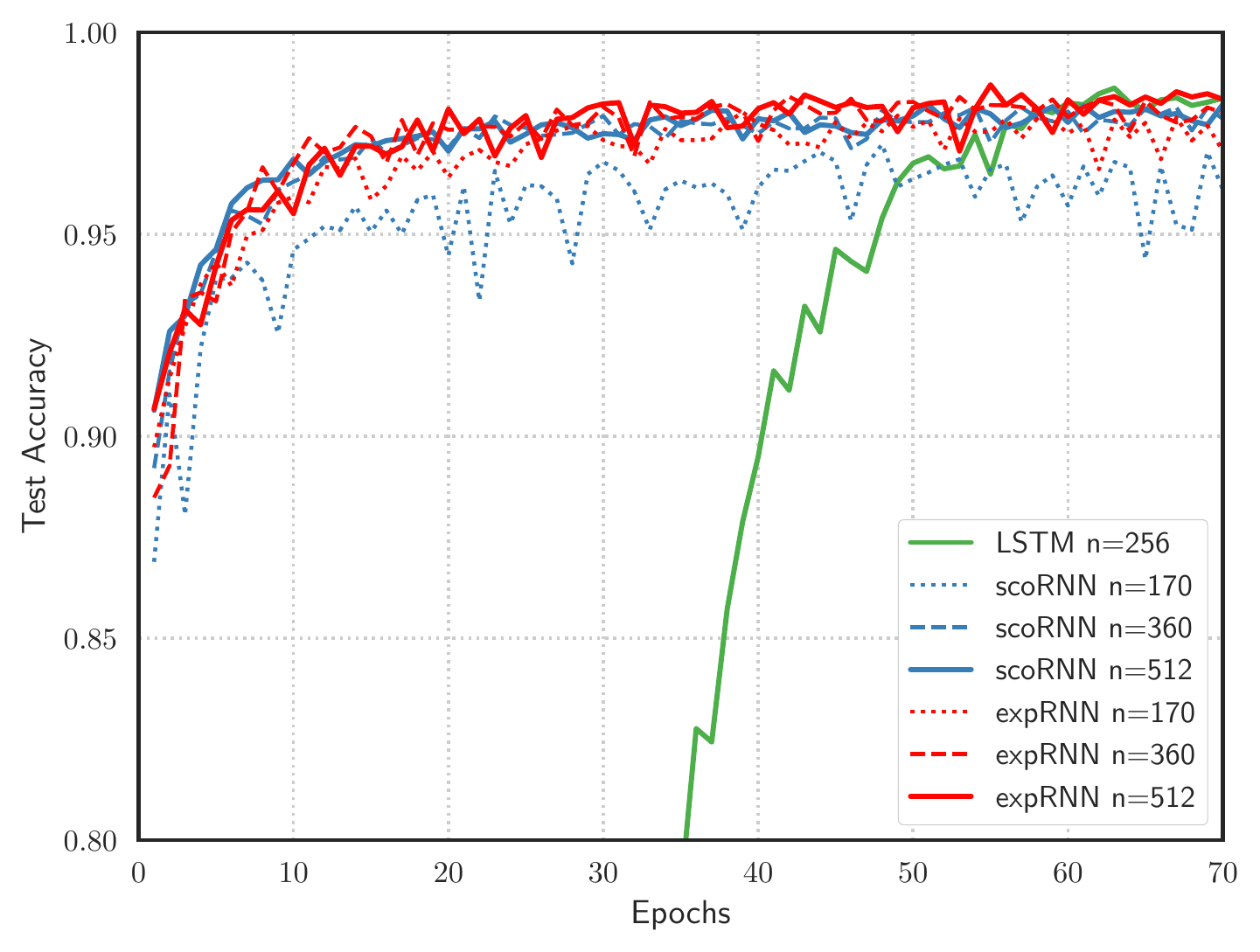}}
  \end{minipage}
  \begin{minipage}[b]{\columnwidth}
      \centerline{\includegraphics[width=0.83\columnwidth]{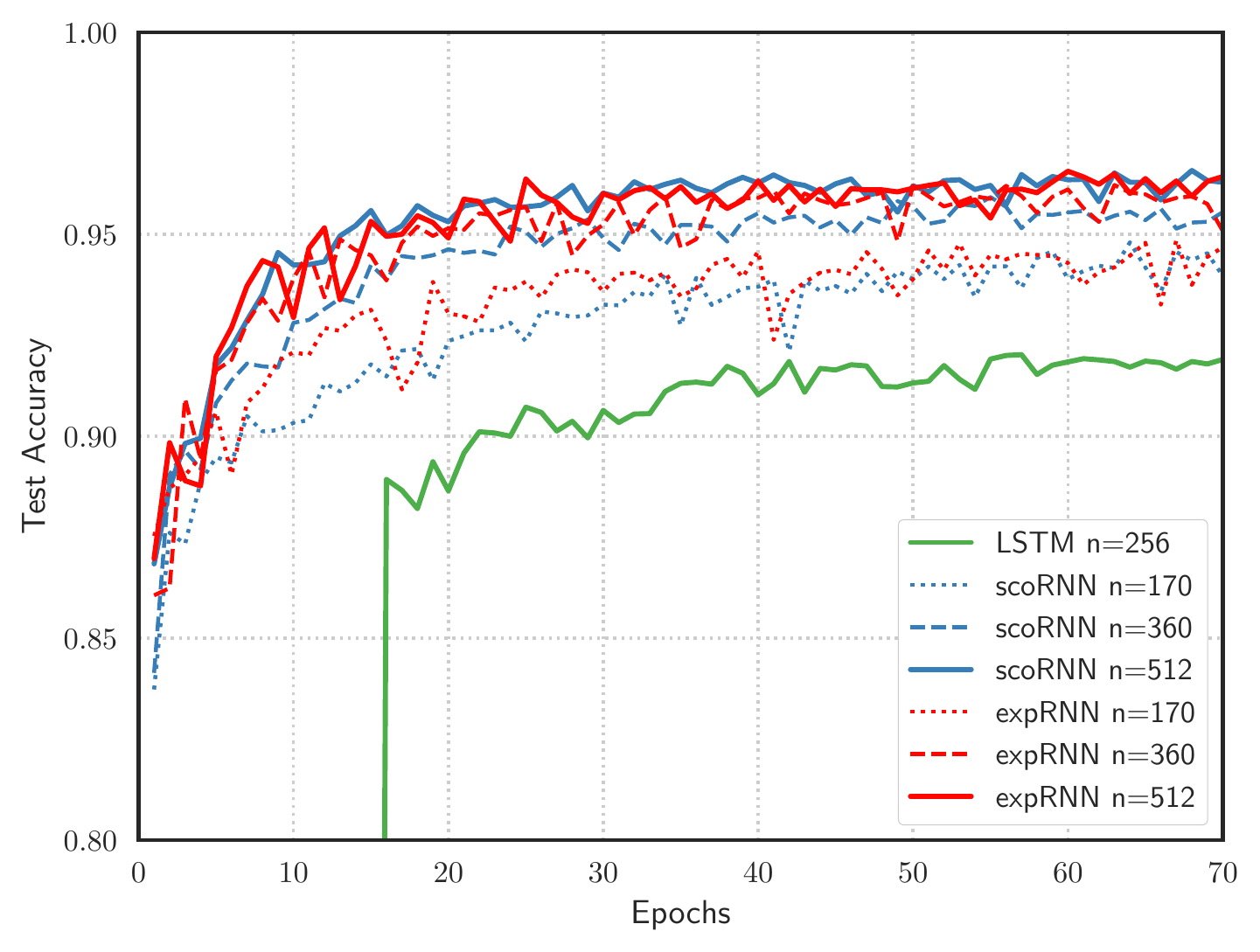}}
  \end{minipage}
  \caption{Test losses for several models on pixel-by-pixel \mnist{} (left) and \pmnist{} (right).}
  \label{fig:mnist}
\end{center}
\vskip -0.2in
\end{figure*}

\subsection{Pixel-by-Pixel \mnist}
In this task we use the \mnist{} dataset of hand-written numbers~\cite{lecun-mnisthandwrittendigit-2010} of images of size $28 \times 28$, only this time the images are flattened and are processed as an array of $784$ pixels, which is treated as a stream that is fed to the \rnn, as described in~\cite{le2015simple}. In the unpermuted task, the stream is processed in a row-by-row fashion, while in the permuted task, a random permutation of the $784$ elements is chosen at the beginning of the experiment, and all the pixels of all the images in the experiment are permuted according to this permutation. The final output of the \rnn{} is processed as the encoding of the number and used to solve the corresponding classification task.

In this experiment we observed that \exprnn{} is able to saturate the capacity of the orthogonal \rnn{} model for this task much faster than any other parametrization, as per~\Cref{tab:mnist}.
We conjecture that coupling the exponential parametrization with an \lstm{} cell or a \gru{} cell would yield a superior architecture. We leave this for future research.

\begin{table}[t]
    \caption{Best test accuracy at the \mnist{} and \pmnist{} tasks.}
\label{tab:mnist}
\vskip 0.15in
\begin{center}
\begin{small}
\begin{sc}
\begin{tabularx}{\columnwidth}{l c X X X}
    \toprule
    Model & n & \# params & \mnist & \pmnist \\
    \midrule
    \midrule
    \exprnn & $170$ & $\approx 16K$ & $0.980$ & $0.949$ \\
    \exprnn & $360$ & $\approx 69K$ & $0.984$ & $0.962$ \\
    \exprnn & $512$ & $\approx 137K$ & $\mathbf{0.987}$ & $\mathbf{0.966}$ \\
    \midrule
    \scornn & $170$ & $\approx 16K$ & $0.972$ & $0.948$ \\
    \scornn & $360$ & $\approx 69K$ & $0.981$ & $0.959$ \\
    \scornn & $512$ & $\approx 137K$ & $0.982$ & $0.965$ \\
    \midrule
    \lstm & $128$ & $\approx 68K$ & $0.819$ & $0.795$ \\
    \lstm & $256$ & $\approx 270K$ & $0.888$ & $0.888$ \\
    \lstm & $512$ & $\approx 1058K$ & $0.919$ & $0.918$ \\
    \midrule
    \rgd & $116$ & $\approx 9K$ & $0.947$ & $0.925$ \\
    \rgd & $512$ & $\approx 137K$ & $0.973$ & $0.947$ \\
    \midrule
    \urnn & $512$ & $\approx 9K$ & $0.976$ & $0.945$ \\
    \urnn & $2170$ & $\approx 69K$ & $0.984$ & $0.953$ \\
    \midrule
    \eurnn & $512$ & $\approx 9K$ & $-$ & $0.937$ \\
    \bottomrule
\end{tabularx}
\end{sc}
\end{small}
\end{center}
\vskip -0.1in
\end{table}

\subsection{\timit{} Speech Dataset}
We performed speech prediction on audio data with our model. We used the \timit{} speech dataset~\cite{garofolo1992timit} which is a collection of real-world speech recordings. The task accounts for predicting the log-magnitude of incoming frames of a short-time Fourier transform (\stft) as it was first proposed in~\cite{wisdom2016full}.

We use the separation in train / test proposed in the original paper, having $3640$ utterances for the training set, a validation set of size $192$, and a test set of size $400$. The validation / test division and the whole preprocessing of the dataset was done according to~\cite{wisdom2016full}. The preprocessing goes as follows: The data is sampled at $8$kHz and then cut into time frames of the same size. These frames are then transformed into the log-magnitude Fourier space and finally, they are normalized according to a per-training set, test set, and validation set basis.

The results for this experiment are shown in~\Cref{tab:timit}. Again, the exponential parametrization beats---by a large margin---other methods of parametrization over the orthogonal group, and also the \lstm{} architecture. The results in~\Cref{tab:timit} are those reported in~\cite{helfrich18a}.

\begin{table}[!ht]
    \caption{Test \mse{} at the end of the epoch with the lowest validation \mse{} for the \timit{} task.}
\label{tab:timit}
\vskip 0.15in
\begin{center}
\begin{small}
\begin{sc}
\begin{tabularx}{\columnwidth}{l c X X X }
    \toprule
    Model & n & \# params & Val. \mse{} & Test \mse{}\\
    \midrule
    \midrule
    \exprnn & $224$ & $\approx 83K$  & $5.34$ & $5.30$ \\ 
    \exprnn & $322$ & $\approx 135K$ & $\mathbf{4.42}$ & $\mathbf{4.38}$ \\   
    \exprnn & $425$ & $\approx 200K$ & $5.52$ & $5.48$ \\ 
    \midrule
    \scornn & $224$ & $\approx 83K$  & $9.26$ & $8.50$ \\
    \scornn & $322$ & $\approx 135K$ & $8.48$ & $7.82$ \\
    \scornn & $425$ & $\approx 200K$ & $7.97$ & $7.36$ \\
    \midrule
    \lstm & $84$ &  $\approx 83K$  & $15.42$ & $14.30$ \\
    \lstm & $120$ & $\approx 135K$ & $13.93$ & $12.95$ \\
    \lstm & $158$ & $\approx 200K$ & $13.66$ & $12.62$ \\
    \midrule
    \eurnn & $158$ & $\approx 83K$  & $15.57$ & $18.51$ \\
    \eurnn & $256$ & $\approx 135K$ & $15.90$ & $15.31$ \\
    \eurnn & $378$ & $\approx 200K$ & $16.00$ & $15.15$ \\
    \midrule
    \rgd & $128$ &  $\approx 83K$  & $15.07$ & $14.58$ \\
    \rgd & $192$ & $\approx 135K$ & $15.10$ & $14.50$ \\
    \rgd & $256$ & $\approx 200K$ & $14.96$ & $14.69$ \\
    \bottomrule
\end{tabularx}
\end{sc}
\end{small}
\end{center}
\vskip -0.1in
\end{table}

As a side note, we must say that the results in this experiment should be interpreted under the following fact: We had access to two of the implementations for the tests for the other architectures regarding this experiment, and neither of them correctly handled sequences with different lengths present in this experiment. We suspect that the other implementations followed a similar approach, given that the results that they get are of the same order. In particular, the implementation released by Wisdom, which is the only publicly available implementation of this experiment, divides by a larger number than it should when computing the average \mse{} of a batch, hence reporting a lower \mse{} than the correct one. Even in this unfavorable scenario, our parametrization is able to get results that are twice as good---the \mse{} loss function is a quadratic function---as those from the other architectures.

\section{Conclusion and Future Work}
In this paper we have presented three main ideas. First, a simple approach based on classic Lie group theory to perform optimization over compact Lie groups, in particular $\SO{n}$ and $\U{n}$, proving its soundness and providing empirical evidence of its superior performance. Second, an implementation trick that allows for the implementation of arbitrary parametrizations at a negligible runtime cost. Finally, we sketched how to improve some existing methods to perform optimization on Lie groups using Riemannian gradient descent.
Any of these three ideas is of independent interest and could have more applications within neural networks.

The investigation of how to couple these ideas with the \lstm{} architecture to improve its performance is left for future work.

Additionally, it could be of interest to see how orthogonal constraints help with learning in deep feed forward networks. In order to make this last point formal, one would have to generalize the results presented here to \emph{homogeneous Riemannian manifolds}, like the Stiefel manifold.

\ifdefined\isarxiv
\else
    \clearpage
\fi

\section*{Acknowledgements}
We would like to thank the help of Prof.~Raphael Hauser and Jaime Mendizabal for the useful conversations, Daniel Feinstein for the proofreading, Prof.~Terry Lyons for the computing power, and Kyle Helfrich for helping us setting up the experiments.

The work of MLC was supported by the Oxford-James Martin Graduate Scholarship and the ``la Caixa'' Banking Foundation (LCF/BQ/EU17/11590067). The work of DMR was supported by EP/N509711/1 from the EPSRC MPLS division, grant No 2053152.

\bibliographystyle{icml2019}
\bibliography{refs}
\clearpage
\onecolumn
\appendix
\section{Riemannian Geometry and Lie Groups}\label{sec:appendix_background}
    In this section we aim to give a short summary of the basics of classical Riemannian geometry and Lie group theory needed for our proofs in following sections. The standard reference for classical Riemannian geometry is do Carmo's book~\cite{do1992riemannian}. An elementary introduction to Lie group theory with an emphasis on concrete examples from matrix theory can be found in~\cite{hall2015lie}.

    \subsection{Riemannian geometry}
    A Riemannian manifold is an $n$-dimensional smooth manifold $\Mani$ equipped with a smooth metric $\scalar{\cdot, \cdot}_p \colon T_p\Mani \times T_p\Mani \to \R$ which is a positive definite inner product for every $p \in \Mani$. We will omit the dependency on the point $p$ whenever it is clear from the context. Given a metric, we can define the length of a curve $\gamma$ on the manifold as $L(\gamma) := \int_a^b \sqrt{\scalar{\gamma'(t), \gamma'(t)}}\,\dif t$. The distance between two points is the infimum of the lengths of the piece-wise smooth curves on $\Mani$ connecting them. When the manifold is connected, this defines a distance function that turns the manifold into a metric space.

    An \emph{affine connection} $\conn$ on a smooth manifold is a bilinear application that maps two vector fields $X,Y$ to a new one $\conn_X Y$ such that it is linear in $X$, and linear and Leibnitz in $Y$.

    Connections give a notion of variation of a vector field along another vector field. In particular, the covariant derivative is the restriction of the connection to a curve. In particular, we can define the notion of parallel transport of vectors. We say that a vector field $Z$ is \emph{parallel} along a curve $\gamma$ if $\conn_{\gamma'} Z = 0$ where $\gamma' \defi \dif \gamma\pa{\deriv{t}}$. Given initial conditions $(p, v) \in T\Mani$ there exists locally a unique parallel vector field $Z$ along $\gamma$ such that $Z(p) = v$. $Z(\gamma(t))$ is sometimes referred to as the \emph{parallel transport of $v$ along $\gamma$}.

    We say that a connection is \emph{compatible with the metric} if for any two parallel vector fields $X, Y$ along $\gamma$, their scalar product is constant. In other words, the connection preserves the angle between parallel vector fields. We say that a connection is \emph{torsion-free} if $\conn_X Y - \conn_Y X = [X,Y] \defi XY - YX$. In a Riemannian manifold, there exists a unique affine connection such that it is compatible with the metric and that is also torsion-free. We call this distinguished connection the \emph{Levi-Civita connection}.

    A geodesic is defined as a curve such that its tangent vectors are covariantly constant along itself, $\conn_{\gamma'}\gamma' = 0$. It is not true in general that given two points in a manifold there exists a geodesic that connects them. However, the \emph{Hopf-Rinow theorem} states that this is indeed the case if the manifold is connected and complete as a metric space. The manifolds that we will consider are all connected and complete.

    At every point $p$ in our manifold we can define the \emph{Riemannian exponential map} $\deffun{\exp_p :T_p\Mani -> \Mani;}$, which maps a vector $v$ to $\gamma(1)$ where $\gamma$ is the geodesic such that $\gamma(0) = p$, $\gamma'(0) = v$. In a complete manifold, another formulation of the \emph{Hopf-Rinow theorem} says that the exponential map is defined on the whole tangent space for every $p \in \Mani$. We also have that the Riemannian exponential map maps diffeomorphically a neighborhood around zero on the tangent space to a neighborhood of the point on which it is defined.

    A map between two Riemannian manifolds is called a (local) isometry if it is a (local) diffeomorphism and its differential respects the metric.

    \subsection{Lie groups}
     A Lie group is a smooth manifold equipped with smooth group multiplication and inverse. Examples of Lie groups are the Euclidean space equipped with its additive group structure and the general linear group $\GLaux$ of a finite dimensional vector space given by the invertible linear endomorphisms of the space equipped with the composition of morphisms. We say that a Lie group is a matrix Lie group if it is a closed subgroup of some finite-dimensional general linear group.

     Lie groups act on themselves via the left translations given by $L_g(x) = gx$ for $g, x \in G$. A vector field $X$ is called \emph{left invariant} if $\pa{\dif L_g}\pa{X} = X \circ L_g$. A left invariant vector field is uniquely determined by its value at the identity of the group. This identification gives us a way to identify the tangent space at a point of the group with the tangent space at the identity. We call the tangent space at the identity \emph{the Lie algebra of $G$} and we denote it by $\g$.

    For every vector $v \in \g$ there exists a unique curve $\deffun{\gamma : \R  -> G;}$ such that $\gamma$ is the integral curve of the left-invariant vector field defined by $v$ such that $\gamma(0) = e$. This curve is a Lie group homomorphism and we call it the \emph{Lie exponential}. It is also the integral curve of the right-invariant vector field with initial vector $v$.

    We say that $c_g(x) = gxg^{-1}$ for $g,x \in G$ is an \emph{inner automorphism of $G$}. Its differential at the identity is the \emph{adjoint representation of $G$}, $\deffun{\Ad : G -> \GL{\g};}$ defined as $\Ad_g(X) \defi \pa{\dif c_g}_e(X)$ for $g \in G$, $X \in \g$. The differential at the identity of $\Ad$ is called the \emph{adjoint representation of $\g$}, $\deffun{\ad : \g -> \End\pa{\g};}$ defined as $\ad_X(Y) \defi \pa{\dif \Ad}_e(X)(Y)$. We say that $\ad_X(Y)$ is the \emph{Lie bracket of $\g$} and we denote it by $[X,Y]$. For a matrix Lie group we have $\Ad_g(X) = gXg^{-1}$ and $\ad_X(Y) = XY - YX$.

    A (complex) representation of a group is a continuous group homomorphism $\deffun{\rho : G -> \GL{n, \C};}$. An injective representation is called \emph{faithful}. The inclusion is a faithful representation for any matrix Lie group. On a compact Lie group, $\rho(g)$ is diagonalizable for every $g \in G$.

    A Riemannian metric on $G$ is said to be bi-invariant if it turns left and right translations into isometries. We have that every compact Lie group admits a bi-invariant metric. An example of a bi-invariant metric on the group of orthogonal matrices with positive determinant $\SO{n}$ is that inherited from $\M{n}$, namely the canonical metric $\scalar{X,Y} = \tr\pa{\trans{X}Y}$. The same happens in the unitary case, but changing the transpose for a conjugate transpose $\scalar{X,Y} = \tr\pa{\adj{X}Y}$. Furthermore, every Lie group that admits a bi-invariant metric is a homogeneous Riemannian manifold---there exists an isometry between that takes any point to any other point---, and hence, complete.

\section{Retractions}\label{sec:retractions}
We take a deeper look into the concept of a retraction, which helps understanding the correctness of the approach used to optimize on $\SO{n}$ and $\U{n}$ presented in~\cite{wisdom2016full,vorontsov2017orthogonality}.

The concept of a retraction is a relaxation of that of the Riemannian exponential map.
\begin{definition}[Retraction]
    A retraction is a map
    \[
        \deffun{r : T\Mani -> \Mani ; (x,v) -> r_x(v)}
    \]
    such that
    \[
        r_x(0) = x \qquad \text{and} \qquad (\dif r_x)_0 = \Id
    \]
    where $\Id$ is the identity map.
\end{definition}
In other words, when $\Mani$ is a Riemannian manifold, $r$ is a first order approximation of the Riemannian exponential map.

It is clear that the exponential map is a retraction. For manifolds embedded in the Euclidean space with the metric induced by that of the Euclidean space, the following proposition gives us a simple way to construct a rather useful family of retractions---those used in projected gradient descent.
\begin{proposition}\label{prop:proj}
    Let $\Mani$ be an embedded submanifold of $\R^n$, then for a differentiable surjective projection $\deffun{\pi : \R^n -> \Mani ;}$, that is, $\pi \circ \pi = \pi$, the map
\[
    \deffun{r : T\Mani -> \Mani ; (x,v) -> \pi(x+v)}
\]
is a retraction, where we are implicitly identifying $T_x \Mani \subset T_x\R^n \iso \R^n$.
\end{proposition}
\begin{proof}
    From $\pi$ being a surjective projection we have that $\pi(x) = x$ for every $x \in \Mani$, which implies the first condition of the definition of retraction.

    Another way of seeing the above is saying that $\pi \vert_{\Mani} = \Id$. This implies that, for every $x \in \Mani$, $(\dif \pi)_x = \Id \vert_{T_x\Mani}$. By the chain rule, since the differential of $v \mapsto x+v$ is the identity as well we get the second condition.
\end{proof}

This proposition lets us see projected Riemannian gradient descent as an specific case of Riemannian gradient descent with a specific retraction. A corollary of this proposition that allows $r$ to be defined just in those vectors of the form $x + v$ with $(x,v) = T\Mani$ lets us construct specific examples of retractions:

\begin{example}[Sphere]
    The function
    \[
        r_x(v) = \frac{x+v}{\norm{x+v}}
    \]
    for $v \in T_x\S^n$ is a retraction.
\end{example}

\begin{example}[Special orthogonal group]
    Recall that for $A \in \SO{n}$,
    \[
        T_A(\SO{n}) = \set{X \in \R^{n \times n} | \trans{A}X + \trans{X}A = 0},
    \]
    then, for an element of $T\SO{n}$ we can define the map given by~\Cref{prop:proj}. In this case,  the projection $\pi(X)$ for a matrix with singular value decomposition $X = U \Sigma \trans{V}$ is $\pi(X) = U\trans{V}$.

    This projection is nothing but the orthogonal projection from $\R^{n \times n}$ onto $\SO{n}$ when equipped with the canonical metric.
\end{example}

The two examples above are examples of orthogonal projections. The manifolds being considered are embedded into a Euclidean space and they inherit its metric. The projections here are orthogonal projections on the ambient space. On the other hand, \Cref{prop:proj} does not require the projections to be orthogonal.

Different examples of retractions can be found in~\cite{absil2009optimization}, Example $4.1.2$.

\section{Comparing Riemannian gradient descent and the exponential parametrization}\label{sec:appendix_comparison}
The proofs in this section are rather technical and general so that they apply to a wide variety of manifolds such as $\SO{n}$, $\U{n}$, or the symplectic group. Even though we do not explore applications of optimizing over other compact matrix Lie groups, we state the results in full generality. Having this in mind, we will first motivate this section with the concrete example that we study in the applications of this paper: $\SO{n}$.

\begin{example}\label{ex:riemannian_gradient_descent_son}
Let $\deffun{f : \M{n} -> \R;}$ be a function defined on the space of matrices. Our task is to solve the problem
\[
    \min_{B \in \SO{n}} f(B).
\]

A simple way to approach this problem would be to apply Riemannian gradient descent to it. Let $A$ be a skew-symmetric matrix an let $B = e^A \in \SO{n}$, where $e^A$ denotes the exponential of matrices. We will suppose that we are working with the canonical metric on $\M{n}$, namely $\scalar{X, Y} = \tr\pa{\trans{X}Y}$. We will denote by $\grad f(B)$ the gradient of the function $f$ at the matrix $B$, and by $\gradmani f(B) \in T_B\SO{n}$ the gradient associated to the restriction $f \vert_{\SO{n}}$ with respect to the induced metric.

Riemannian gradient descent works by following the geodesic defined by the direction $- \gradmani f(B)$ at the point $B$. In the words of the Riemannian exponential map at $B$, if we have a learning rate $\eta > 0$, the update rule will be given by
\[
    B \gets \exp_B\pa{-\eta\gradmani f(B)}.
\]

The tangent space to $\SO{n}$ at a matrix $B$ is
\[
    T_B\SO{n} = \set{X \in \M{n} | \trans{B}X + \trans{X}B = 0}
\]
and it is easy to check that the orthogonal projection with respect to the canonical metric onto this space is given by
\[
    \deffun{\pi_B : \M{n} -> T_B\SO{n}; X -> \tfrac{1}{2}\pa{X - B\trans{X}B}}
\]
Since the gradient on the manifold is just the tangent component of the gradient on the ambient space, we have that
\[
    \gradmani f(B) = \tfrac{1}{2}\pa{\grad f(B) - B\trans{\grad f(B)}B},
\]
and since multiplying by an orthogonal matrix constitutes an isometry, in order to compute the Riemannian exponential map we can transport the vector from $T_BG$ to $T_{\I}G$, compute the exponential at the identity using the exponential of matrices and then transport the result back. In other words,
\[
    \exp_B(X) = B\exp\pa{\trans{B}X} \qquad \forall X \in T_B\SO{n}.
\]

Putting everything together, the Riemannian gradient descent update rule for $\SO{n}$ is given by
\[
    B \gets e^A e^{-\eta \trans{B} \gradmani f(B)}.
\]

The other update rule that we have is the one given by the exponential parametrization
\[
    B \gets e^{A - \eta \grad\pa{f \circ \exp}(A)}.
\]
This is nothing but the gradient descent update rule applied to the problem
\[
    \min_{A \in \Skew{n}} f\pa{\exp\pa{A}}.
\]

Both of these rules follow geodesic flows for two metrics whenever $A$ is in a neighborhood of the identity on which $\exp$ is a diffeomorphism (\cf \Cref{sec:appendix_fundamental_domain}). A natural question that arises is whether these metrics are the same. A less restrictive question would be whether this second optimization procedure defines a retraction or whether their gradient flow is completely different.

In the sequel, we will see that for $\SO{n}$ these two optimization methods give raise to two rather different metrics. We will explicitly compute the quantity $\grad\pa{f \circ \exp}(A)$, and we will give necessary and sufficient conditions equivalent under which these two optimization methods agree.
\end{example}

\subsection{Optimization on Lie Groups with Bi-invariant Metrics}
In this section we expose the theoretical part of the paper. The first part of this section is
classic and can be found, for example, in Milnor~\cite{milnor1976curvatures}. We present it here for
completeness.
The results~\Cref{prop:gradient} and~\Cref{thm:retract} are novel.

\begin{remark}
    Throughout this section the operator $\pa{-}^\ast$ will have two different meanings. It can be either the pullback of a form along a function or the adjoint of a linear operator on a vector space with an inner product. Although the two can be distinguished in many situations, we will explicitly mention to which one we are referring whenever it may not be clear from the context. Note that when we are on a matrix Lie group equipped with the product $\scalar{X,Y} = \tr\pa{\adj{X}Y}$, the adjoint of a linear operator is exactly its conjugate transpose, hence the notation.

    When we deal with an abstract group we will denote the identity element as $e$. If the group is a matrix Lie group, we will sometimes refer to it as $\I$.
\end{remark}

We start by recalling the definition of our object of study.
\begin{definition}[Invariant metric on a Lie Group]
    A Riemannian metric on a Lie group $G$ is said to be left (resp.\ right) invariant if it makes left (resp.\ right) translations into isometries. Explicitly, it is so if for every $g \in G$, and the metric $\alpha$ we have that $\alpha_g = L^\ast_{g^{-1}} \alpha_e$ (resp. $\alpha_g = R^\ast_{g^{-1}}\alpha_e$).

    A bi-invariant metric is a metric that is both left and right-invariant.
\end{definition}

We can construct a bi-invariant metric on a Lie group by using the \emph{averaging trick}.
\begin{proposition}[Bi-invariant metric on compact Lie groups]
    A compact Lie group $G$ admits a bi-invariant metric.
\end{proposition}
\begin{proof}
    Let $n$ be the dimension of $G$ and let $\mu_e$ be a non-zero $n$-form at $\g$. This form is unique up to a multiplicative constant. We can then extend it to the whole $G$ by pulling it back along $R_g$ defining $\mu_g \defi R^\ast_{g^{-1}} \mu_e$. This makes it into a right-invariant $n$-form on the manifold, which we call the \emph{right Haar measure}.

    Let $\pa{\cdot, \cdot}$ be an inner product on $\g$. We can turn this inner product into an $\Ad$-invariant inner product on $\g$ by averaging it over the elements of the group using the right Haar measure
    \[
        \scalar{u, v} = \int_G \pa{\Ad_g(u), \Ad_g(v)}\,\mu(\dif g).
    \]
    Note that this integral is well defined since $G$ is compact.
    The $\Ad$-invariance follows from the right-invariance of $\mu$
    \[
        \scalar{\Ad_h\pa{u}, \Ad_h\pa{v}} = \int_G \pa{\Ad_{gh}(u), \Ad_{gh}(v)}\,\mu(\dif g) = \scalar{u, v}.
    \]

    Finally, we can extend this product to the whole group by pulling back the inner product along $L_g$, that is, if we denote the metric by $\alpha$, $\alpha_g = L^\ast_{g^{-1}}\alpha_e$. This automatically makes it into a left-invariant metric. But since it is $\Ad$-invariant at the identity, we have that for every $g, h \in G$
    \[
        R^\ast_{g} \alpha_{hg} = R^\ast_g L^\ast_{g^{-1}h^{-1}} \alpha_e = \Ad^\ast_{g^{-1}} L^\ast_{h^{-1}} \alpha_e = L^\ast_{h^{-1}} \alpha_e = \alpha_h
    \]
    and the metric is also right-invariant, finishing the proof.
\end{proof}

If the group is abelian, the construction above is still valid without the need of the averaging trick, since $\Ad$ is the identity map, so every inner product is automatically $\Ad$-invariant.

It turns out that these examples and their products exhaust all the Lie groups that admit a bi-invariant metric. We include this result for completeness, even though we will not use it.
\begin{theorem}[Classification of groups with bi-invariant metrics]
    A Lie group admits a bi-invariant metric if and only if it is isomorphic to $G \times H$ with $G$ compact and $H$ abelian.
\end{theorem}
\begin{proof}
    \cite{milnor1976curvatures} Lemma 7.5.
\end{proof}

Lie groups, when equipped with a bi-invariant metric are rather amenable from the Riemannian geometry perspective. This is because it is possible to reduce many computations on them to matrix algebra, rather than the usual systems of differential equations that one encounters when dealing with arbitrary Riemannian manifolds.

The following proposition will come in handy later.
\begin{lemma}~\label{lemma:ad_unitary}
    If an inner product on $\g$ is $\Ad$-invariant then
    \[
        \scalar{Y, \ad_X(Z)} = -\scalar{\ad_{X}(Y), Z} \qquad \forall X, Y, Z \in \g.
    \]
    In other words, the adjoint of the map $\ad_X$ with respect to the inner product is $-\ad_X$. We say that $\ad$ is skew-adjoint and we write $\ad^\ast_X = -\ad_X$.
\end{lemma}
\begin{proof}
    We have that, by definition
    \[
        \ad_X(Y) = \deriv{t}\pa{\Ad_{\exp\pa{tX}}(Y)}\mid_0,
    \]
    so, deriving the equation
    \[
        \scalar{\Ad_{\exp\pa{tX}}(Y), \Ad_{\exp\pa{tX}}(Z)} = \scalar{Y, Z}
    \]
    with respect to $t$ we get the result.
\end{proof}

With this result in hand, we can prove a rather useful relation between the geometry of the Lie group and its algebraic structure.
\begin{proposition}\label{prop:connection_curvature}
    Let $G$ be a Lie group equipped with a bi-invariant metric. If $X, Y$ are left-invariant vector fields, we have that their Levi-Civita connection and their sectional curvature are given by
    \begin{gather*}
        \conn_X Y = \frac{1}{2}[X,Y]\\
        \kappa\pa{X,Y} = \frac{1}{4}\norm{[X,Y]}^2.
    \end{gather*}
    The sectional curvature formula holds whenever $X$ and $Y$ are orthonormal.
\end{proposition}
\begin{proof}
    For left-invariant vector fields $X,Y,Z$, the Koszul formula gives
    \[
        \scalar{\conn_X Y, Z} = \frac{1}{2}\pa{X\scalar{Y,Z} + Y\scalar{X,Z} - Z\scalar{X,Y} + \scalaraux{[X,Y]}{Z} -\scalaraux{[X,Z]}{Y} - \scalaraux{[Y,Z]}{X}}.
    \]
    The three first terms on the right vanish, since the angle between invariant vector fields is constant. Reordering the last three terms, using~\Cref{lemma:ad_unitary}, the fact that the Lie bracket is antisymmetric, and since invariant vector fields form a basis of the Lie algebra, the formula for the connection follows. Now, the curvature tensor is given by
    \[
        R(X,Y)Z = \conn_X \conn_Y Z - \conn_Y\conn_X Z - \conn_{[X,Y]}Z = \frac{1}{4}\pa{[X, [Y,Z]] - [Y, [X, Z]] - 2[[X,Y],Z]} = \frac{1}{4}[Z, [X,Y]].
    \]
    So the sectional curvature for $X, Y$ orthonormal is given by
    \[
        \kappa(X,Y) = \scalaraux{R(X,Y)Y}{X} = \frac{1}{4}\norm{[X,Y]}^2.\qedhere
    \]
\end{proof}

On a matrix Lie group equipped with a metric we have three different notions of exponential maps, namely the Lie exponential map, the Riemannian exponential map and the exponential of matrices. We will now show that if we consider the Riemannian exponential map at the identity element, these three concepts agree whenever the metric is bi-invariant.

\begin{proposition}[Equivalence of Riemannian and Lie exponential]
    Let $G$ be a Lie group equipped with a bi-invariant metric. Then, the Riemannian exponential at the identity $\exp_e$ and the Lie exponential $\exp$ agree.
\end{proposition}
\begin{proof}
    Fix a vector $X_e \in \g$ and consider the curve $\gamma(t) = \exp(tX_e)$. This curve is the integral curve of the invariant vector field defined by $X_e$, this is $\gamma'(t) = X(\gamma(t))$. For this reason, by~\Cref{prop:connection_curvature}
    \[
        \conn_{\gamma'}\gamma' = \frac{1}{2}[X,X] = 0.
    \]
    so $\exp(tX_e)$ is a geodesic and the result readily follows.
\end{proof}

\begin{proposition}[Equivalence of Lie exponential and exponential of matrices]
    Let $G$ be a matrix Lie group, that is, a closed subgroup of $\GL{n, \C}$. Then the matrix exponential $\exp_M$ and the Lie exponential $\exp$ agree.
\end{proposition}
\begin{proof}
    The matrix exponential $\exp_{tX}$ can be expressed as the solution of the matrix differential equation
    \[
        \gamma'(t) = X\gamma(t) \qquad \gamma(0) = \I,\, t \in \R
    \]
    for $X \in \C^{n \times n} = \mathfrak{gl}(n, \C)$. This is exactly the equation that defines the Lie exponential map as the integral curve of a right-invariant vector field, that is, the Lie exponential.
\end{proof}

Finally, all these equivalences give a short proof of the fact that the Lie exponential map is surjective on a connected Lie group with a bi-invariant metric.
\begin{theorem}[Lie exponential surjectivity]\label{thm:surjectivity}
    Let $G$ be a connected Lie group equipped with a bi-invariant metric. The Lie exponential is surjective.
\end{theorem}
\begin{proof}
    As the Lie exponential is defined in the whole Lie algebra, so is the map $\exp_e$. Since the metric is bi-invariant, we have that at a point $(g, v) \in TG$, $\exp_g(v) = L_g\pa{\exp_e\pa{\pa{\dif L_{g^{^-1}}}_g\pa{v}}}$ and since left-translations are diffeomorphisms, the Riemannian exponential is defined in the whole tangent bundle. Therefore, by the Hopf-Rinow theorem, this implies that there exists a geodesic between any two points. Since the geodesics starting at the identity are given by the curves $\gamma(t) = \exp(tX_e)$ for $X_e \in \g$, the result follows.
\end{proof}

Now that we have all the necessary tools, we shall return to the problem of studying the metric induced by the exponential parametrization.

As we have seen, the problem that we are interested in solving is
\[
    \min_{g \in G} f(g),
\]
where $G$ is a matrix Lie group equipped with an bi-invariant metric. The exponential parametrization maps this problem back to the Lie algebra
\[
    \min_{X \in \g} f(\exp(X)).
\]
Since $\g$ is a vector space, putting a basis on it we have that we can use all the classical toolbox developed for Euclidean spaces to approach a solution for this problem. In particular, in the context of neural networks, we are interested in studying first-order optimization methods that approach a solution to this problem, in particular, gradient descent methods. The gradient descent update step for this problem with learning rate $\eta > 0$ is given by
\[
    X \gets X - \eta\grad(f\circ \exp)(X),
\]
where the gradient is defined with respect to the metric, that is, it is the vector such that
\[
    \dif\pa{f \circ \exp}_X(Y) = \scalar{\grad\pa{f \circ \exp}\pa{X}, Y}.
\]

To study this optimization method, we first have to make sense of the gradient $\grad\pa{f \circ \exp}(X)$. To do so, we will make use of the differential of the exponential map.

\begin{proposition}\label{prop:diferential_exponential}
    The differential of the exponential map on a matrix Lie group is given by the formula
    \[
        \pa{\dif \exp}_X(Y) = e^X\sum_{k=0}^\infty \frac{\pa{-\ad_X}^k}{(k+1)!}\pa{Y} \qquad \forall X,Y \in \g.
    \]
\end{proposition}
\begin{proof}
    \cite{hall2015lie} Theorem 5.4.
\end{proof}
An analogous formula still holds in the general case, but the proof is more delicate. The powers of the adjoint representation are to be thought as the composition of endomorphisms on $\g$. For this reason, this formula can be also expressed as
\[
    \pa{\dif \exp}_X(Y) = e^X\pa{Y - \frac{1}{2}[X,Y] + \frac{1}{6}[X, [X, Y]] - \dots}.
\]
Yet another way of looking at this expression is by defining the function
\[
    \deffun{\phi : \End\pa{\g} -> \End\pa{\g}; X -> \frac{1-e^{-X}}{X}}
\]
so that
\[
    \pa{\dif \exp}_X = \dif L_{e^X} \circ \phi\pa{\ad_X}.
\]
In this case, the fraction that defines $\phi$ is just a formal expression to refer to the formal series defined in~\Cref{prop:diferential_exponential}.

From this we can compute the gradient of $f \circ \exp$.
\begin{proposition}\label{prop:gradient}
    Let $\deffun{f : G -> \R;}$ be a function defined on a matrix Lie group equipped with a bi-invariant metric. For a matrix $A \in \g$ let $B = e^A$. We have
    \[
    \grad\pa{f \circ \exp}(A) = B\pa{\dif \exp}_{-A}\pa{B^{-1} \grad f\pa{B}} = \sum_{k=0}^\infty \frac{\pa{\ad_A}^k}{(k+1)!}\pa{e^{-A}\grad f\pa{B}} .
    \]
\end{proposition}
\begin{proof}
    Let $U \in \g$. By the chain rule, we have
    \[
        \pa{\dif \pa{f \circ \exp}}_A(U) = \pa{\dif f}_B \circ \pa{\dif \exp}_A (U).
    \]
    In terms of the gradient of $f$ with respect to the metric this is equivalent to
    \begin{align*}
        \pa{\dif \pa{f \circ \exp}}_A(U) &= \scalar{\grad f(B), \pa{\dif \exp}_A (U)} \\
                                         &= \scalar{\pa{\dif \exp}^\ast_A\pa{\grad f(B)}, U}
    \end{align*}
    which gives
    \[
        \grad\pa{f \circ \exp}(A) = \pa{\dif \exp}^\ast_A\pa{\grad f(B)}.
    \]
    Now we just have to compute the adjoint of the differential of the exponential function. This is now simple since
    \begin{align*}
        \pa{\dif \exp}^\ast_A &= \pa{\dif L_{e^A} \circ \phi\pa{\ad_A}}^\ast \\
                              &= \phi\pa{\ad_A}^\ast \circ \dif L_{e^{-A}} \\
                              &= \phi\pa{\ad_A^\ast} \circ \dif L_{e^{-A}} \\
                              &= \phi\pa{\ad_{-A}} \circ \dif L_{e^{-A}},
    \end{align*}
    where the second equality follows from the product being left-invariant, the third one from $\phi$ being analytic and the last one from~\Cref{lemma:ad_unitary}.
\end{proof}

Now we can explicitly define the update rule for the exponential parametrization
\[
    \deffun{\widehat{r} : TG -> G ; \pa{e^A, U} -> \exp\pa{A + \phi\pa{\ad_{-A}}\pa{e^{-A}U}}}.
\]
We can then study the gradient flow induced by the exponential parametrization by means of $\widehat{r}$. If $\widehat{r}$ were a retraction, then the flow induced by the exponential parametrization would have similar properties as that of Riemannian gradient descent, as shown in~\cite{boumal2016global}. It turns out that the exponential parametrization induces a different flow.

\begin{theorem}\label{thm:retract}
    Let $G$ be a connected matrix Lie group equipped with a bi-invariant metric. The function $\widehat{r}$ is a retraction if and only if $G$ is abelian.
\end{theorem}
\begin{proof}
    It is clear that $\widehat{r}_g(0) = g$ for every $g \in G$. Let $A \in \g$ and $B = e^A$ and let $U \in T_BG$. By the chain rule we have that
    \[
        \pa{\dif \widehat{r}_B}_0(U) = \pa{\dif \exp}_A\pa{\pa{\dif \exp}_A^\ast\pa{U}}.
    \]
    The map $\widehat{r}$ is a retraction if and only if
    \[
        \pa{\dif \exp}_A\pa{\pa{\dif \exp}^\ast_A\pa{U}} = U
    \]
    holds for every $U \in T_BG$. This is equivalent to having
    \[
        \scalar{\pa{\dif \exp}_A\pa{\pa{\dif \exp}_A^\ast\pa{U}}, H} = \scalar{U, H}
    \]
    for every $H \in T_BG$. Taking adjoints and since the metric is left-invariant, using the formula for the adjoint of the differential of the exponential map computed in~\Cref{prop:gradient}, or equivalently
    \[
        \scalar{\pa{\dif \exp}_{-A}\pa{X}, \pa{\dif \exp}_{-A}\pa{Y}} = \scalar{X, Y} \qquad \forall X, Y\in \g.
    \]
    In other words, $\widehat{r}$ is a retraction if and only if the Lie exponential map is a local isometry.

    Now, the Lie exponential maps $\g$ into $G$, but $\g$ equipped with its metric is flat, so it has constant sectional curvature zero. On the other hand, the sectional curvature of $G$ is given by $\kappa(X,Y) = \frac{1}{4}\norm{[X,Y]}^2$. Recall that a Lie group is abelian if and only if its Lie bracket is zero.

    If the Lie bracket is zero, $\pa{\dif \exp}_A = \pa{\dif L_{e^A}}_e$, and it is an isometry.

    Conversely, if it is an isometry, it preserves the sectional curvature, so the Lie bracket has to be identically zero, hence the group is Abelian.
\end{proof}

In the abelian case, we do not only have that $\widehat{r}$ is a retraction, but also that the update rule for the exponential parametrization agrees with that of Riemannian gradient descent. Recall that the Riemannian gradient descent rule for a gradient $U$ and a step-size $\eta$ is given by
\[
    e^A e^{-\eta e^{-A}U}.
\]
On an abelian group we have that $e^X e^Y = e^{X + Y}$. Furthermore, since the adjoint representation is zero, $\pa{\dif \exp}_A(U) = e^AU$. Putting these two things together we have
\[
    \widehat{r}(e^A, -\eta U) = e^{A -\eta e^{-A} U} = e^A e^{-\eta e^{-A} U}.
\]

\section{Maximal Normal Neighborhood of the Identity}\label{sec:appendix_fundamental_domain}
\begin{definition}[Normal Neighborhood]
    Let $V \subset T_p\Mani$ be a neighborhood of $0$ such that the Riemannian exponential map $\exp_p$ is a diffeomorphism. Then we say that $\exp_p V$ is a \emph{normal neighborhood} of $p$.
\end{definition}

Given that on a matrix Lie group $\pa{\dif\exp}_\I = \Id$, by the inverse function theorem, there exists a normal neighborhood around the identity matrix. In this section we will prove that the maximal open normal neighborhood of $\SO{n}$ (resp.\ $\U{n}$) covers almost all the group. By almost all the group we mean that the closure of the normal neighborhood is equal to the group. We do so by studying at which points we have that the map $\exp$ is no longer an immersion, or in other words, we look at the points $A \in \g$ at which $\det\pa{\pa{\dif \exp}_A} = 0$. We will prove so for the group $\GL{n, \C}$, so that the arguments readily generalize to any matrix Lie group.

Recall the definition of the matrix-valued function $\phi$ defined on the space of endomorphisms of the Lie algebra of $\GL{n, \C}$. Specifically, since $\mathfrak{gl}\pa{n, \C} \iso \C^{n \times n}$, we have that $\End\pa{\mathfrak{gl}\pa{n, \C}} \iso \C^{n^2 \times n^2}$
\[
    \deffun{\phi : \C^{n^2 \times n^2} -> \C^{n^2 \times n^2}; A -> \frac{1-e^{-A}}{A}=\sum_{k=0}^\infty \frac{\pa{-A}^k}{\pa{k+1}!}}.
\]
Using this function, we can factorize the differential of the exponential function on a matrix Lie group as
\[
    \pa{\dif \exp}_A = e^A\phi\pa{\ad_A}.
\]

Let us now compute the maximal normal neighborhood of the identity. This result is classic, but the proof here is a simplification of the classical one using an approximation argument.
\begin{theorem}\label{prop:domain_of_definition}
    Let $G$ be a compact and connected matrix Lie group. The exponential function is analytic, with analytic inverse on a bounded open neighborhood of the origin given by
    \[
        U = \set{A \in \g | \abs{\Im\pa{\lambda_i(A)}} < \pi}.
    \]
\end{theorem}
\begin{proof}
    Given that $L_{e^A}$ is a diffemorphism, we are interested in studying when the matrix defined by the function $\phi\pa{\ad_A}$ stops being full-rank and when is it injective.

First, note that if the eigenvalues of $A$ are $\lambda_i$, then the eigenvalues of $g(A)$ with $g$ a complex analytic function well-defined on $\set{\lambda_i}$ are $\set{g(\lambda_i)}$. This is clear for diagonalizable matrices. Since these are dense in $\C^{n \times n}$, given that eigenvalues are continuous functions of the matrix, it readily generalizes to arbitrary matrices.

Let $A \in \C^{n^2 \times n^2}$ and let $\lambda_{i,j}$ for $1 \leq i, j \leq n$ be its eigenvalues. Then, $\phi$ is non-singular when $\phi(\lambda_{i,j}) \neq 0$ for every $\lambda_{i,j}$. Equivalently, when $\lambda_{i,j} \neq 2\pi ki$ for $k \in \Z \backslash \set{0}$.

Let us now compute the eigenvalues of $\ad_A$ using the same trick as above. Let $A \in \C^{n \times n}$ and suppose that it is diagonalizable with eigenvalues $\set{\lambda_i}$. Let $u_i$ be the eigenvectors of $A$ and $v_i$ the eigenvectors of $\trans{A}$---which is also diagonalizable with the same eigenvalues---. Since
\[
    \ad_A(u_i \otimes v_j) = \pa{\lambda_i - \lambda_j}u_i \otimes v_j
\]
we have that $\set{u_i \otimes v_j}$ are the eigenvectors of $\ad_A$ with eigenvalues $\lambda_{i,j} \defi \lambda_i - \lambda_j$. Now, using the same continuity argument as above have that these are the eigenvalues of $\ad_A$ for every $A \in \C^{n \times n}$.

From all this, we draw that $\pa{\dif \exp}_A$ is singular whenever $A$ has two eigenvalues that differ by a non-zero integer multiple of $2\pi i$.

Finally, on a compact matrix Lie group every matrix is diagonalizable over the complex numbers, so the exponential acts on the eigenvalues, but the complex variable function $e^z$ is injective on $\set{z \in \C | \abs{\Im\pa{z}} < \pi}$, so the Lie exponential is injective on this domain as well.
\end{proof}

Let us look at the particular cases that we are interested in. If we set $G = \SO{n}$, we have that its Lie algebra are the skew-symmetric matrices. Skew-symmetric matrices have purely imaginary eigenvalues. Furthermore, since they are real matrices, their eigenvalues come in conjugate pairs. As such, we have that the exponential map is singular on every matrix in the boundary of the set $U$ defined in~\Cref{prop:domain_of_definition}.

Special orthogonal matrices are those matrices which are similar to block-diagonal matrices with diagonal blocks of the form
\[
B =
\begin{pmatrix}
    \cos\theta & \sin\theta\\
    -\sin\theta & \cos\theta
\end{pmatrix}
\]
for $\theta \in (-\pi, \pi]$. On $\SO{2n+1}$, there is an extra block with a single $1$.

Similarly, skew-symmetric matrices are those matrices which are similar to block-diagonal matrices with diagonal blocks of the form
\[
A =
\begin{pmatrix}
    0 & \theta\\
    -\theta & 0
\end{pmatrix}.
\]
On $\mathfrak{so}\pa{2n+1}$ there is an extra block with a single zero.

This outlines an elementary proof of the fact that the Lie exponential is surjective on $\SO{n}$ and $\U{n}$.

In both cases this shows that the boundary of $U$ has measure zero and that $f(\overline{U}) = G$.

\begin{remark}
The reader familiar with Lie group theory will have noticed that this proof is exactly the standard one for the surjectivity of the exponential map using the Torus theorem, where one proves that all the maximal tori in a compact Lie group are conjugated and that every element of the group lies in some maximal torus, arriving then to the same conclusion but in much more generality.
\end{remark}

\section{Hyperparameters for the Experiments}\label{sec:hyperparameters}

The batch size across all the experiments was $128$. The learning rates for the orthogonal parameters are $10$ times less those for the non-orthogonal parameters. We fixed the seed of both Numpy and Pytorch to be $5544$ for all the experiments for reproducibility. This is the same seed that was used in the experiments in~\cite{helfrich18a}. In~\Cref{tab:hyperparam} we refer to the optimizer and learning rate for the non-orthogonal part of the neural network simply as optimizer and learning rate.

\begin{table}[!ht]
    \centering
    \caption{Hyperparameters for the Experiments in~\Cref{sec:experiments}.}
    \label{tab:hyperparam}
    \begin{tabular}{ll|cccc}
        \toprule
        Dataset & Size & Optimizer & Learning Rate & Orthogonal optimizer & Orthogonal Learning Rate \\
        \midrule
        \midrule
        Copying Problem $L=1000$ &
        \multirow{2}{*}{$190$} &
        \multirow{2}{*}{\rmsprop} &
        $2\cdot 10^{-4}$ &
        \multirow{2}{*}{\rmsprop} &
        $2\cdot 10^{-5}$ \\
        Copying Problem $L = 2000$ & & & $2\cdot 10^{-4}$ & & $2\cdot 10^{-5}$ \\
        \midrule
        \multirow{3}{*}{\mnist} &
        $170$ &
        \multirow{6}{*}{\rmsprop} &
        $7\cdot 10^{-4}$ &
        \multirow{6}{*}{\rmsprop} &
        $7\cdot 10^{-5}$ \\
        & $360$ & & $5\cdot 10^{-4}$ & & $5\cdot 10^{-5}$ \\
        & $512$ & & $3\cdot 10^{-4}$ & & $3\cdot 10^{-5}$ \\
        \multirow{3}{*}{\pmnist} &
        $170$ &
        &
        $10^{-3}$ &
        &
        $10^{-4}$ \\
        & $360$ & & $7\cdot 10^{-4}$ & & $7\cdot 10^{-5}$\\
        & $512$ & & $5\cdot 10^{-4}$ & & $5\cdot 10^{-5}$ \\
        \midrule
        \multirow{3}{*}{\timit} &
        $224$ &
        \multirow{3}{*}{\adam} &
        $10^{-3}$ &
        \multirow{3}{*}{\rmsprop} &
        $10^{-4}$ \\
        & $322$ & & $7\cdot 10^{-4}$ & & $7\cdot 10^{-5}$ \\
        & $425$ & & $7\cdot 10^{-4}$ & & $7\cdot 10^{-5}$ \\
        \bottomrule
    \end{tabular}
\end{table}

\end{document}